\documentclass[preprint,3p,times,review]{elsarticle}

\usepackage{amssymb}
\usepackage{amsmath}

\usepackage{graphicx}
\usepackage{amsmath}
\usepackage{booktabs}
\usepackage{algorithm}
\usepackage{algorithmic}
\usepackage{url}
\usepackage{multirow}
\usepackage{graphicx}
\usepackage{epstopdf}
\usepackage{booktabs}
\usepackage{amsthm}

\newtheorem{lemma}{Lemma}

\newtheorem{theorem}{Theorem}

\usepackage{enumitem}
\usepackage{array}
\usepackage{subfigure}
\usepackage{color}
\usepackage{xcolor}
\usepackage{caption}

\journal{Expert Systems With Applications}

\begin{document}

\begin{frontmatter}

\title{Learning Unified Distance Metric for Heterogeneous Attribute Data Clustering} 

\author[gdut]{Yiqun Zhang} \ead{yqzhang@gdut.edu.cn}  
\author[gdut]{Mingjie Zhao} \ead{zmingjie775@gmail.com}
\author[xmu]{Yizhou Chen}\ead{imboaba@gmail.com}
\author[xmu]{Yang Lu\corref{cor4}} \ead{luyang@xmu.edu.cn}
\author[hkbu]{Yiu-ming Cheung} \ead{ymc@comp.hkbu.edu.hk}  

\cortext[cor4]{Corresponding author}

\affiliation[gdut]{
    organization={School of Computer Science and Technology},
    addressline={Guangdong University of Technology}, 
    city={Guangzhou},
    postcode={510006}, 
    state={Guangdong},
    country={China}
}
\affiliation[xmu]{
    organization={Key Laboratory of Multimedia Trusted Perception and Efficient Computing, Ministry of Education of China},
    addressline={Xiamen University}, 
    city={Xiamen},
    postcode={361102}, 
    state={Fujian},
    country={China}
}


\affiliation[hkbu]{
    organization={Department of Computer Science},
    addressline={Hong Kong Baptist University}, 
    state={Hong Kong SAR},
    country={China}
}


\begin{abstract}
Datasets composed of numerical and categorical attributes (also called mixed data hereinafter) are common in real clustering tasks. Differing from numerical attributes that indicate tendencies between two concepts (e.g., high and low temperature) with their values in well-defined Euclidean distance space, categorical attribute values are different concepts (e.g., different occupations) embedded in an implicit space. Simultaneously exploiting these two very different types of information is an unavoidable but challenging problem, and most advanced attempts either encode the heterogeneous numerical and categorical attributes into one type, or define a unified metric for them for mixed data clustering, leaving their inherent connection unrevealed. This paper, therefore, studies the connection among any-type of attributes and proposes a {novel} Heterogeneous Attribute Reconstruction and Representation (HARR) learning paradigm accordingly for cluster analysis. The paradigm transforms heterogeneous attributes into a homogeneous status for distance metric learning, and integrates the learning with clustering to automatically adapt the metric to different clustering tasks. {Differing from most existing works that directly adopt defined distance metrics or learn attribute weights to search clusters in a subspace. We propose to project the values of each attribute into unified learnable multiple spaces to more finely represent and learn the distance metric for categorical data.} HARR is parameter-free, convergence-guaranteed, and can more effectively self-adapt to different sought number of clusters $k$. Extensive experiments illustrate its superiority in terms of accuracy and efficiency.

\end{abstract}

\begin{highlights}
\item This research introduces a novel perspective on linking numerical, nominal, and ordinal attributes by exploring the intrinsic semantic concepts they represent, enhancing the understanding of heterogeneous attributes in mixed datasets.
\item A new projection-based method is proposed to transform heterogeneous distance spaces of various attributes into homogeneous ones, providing a robust basis for mixed data analysis.
\item Two algorithms that operate beyond traditional hyper-parameter tuning were proposed, enabling cluster searches in attribute subspaces and increasing learning flexibility.
\end{highlights}


\end{frontmatter}

\section{Introduction}
\label{sec1}

Mixed datasets comprising numerical and categorical attributes are widespread in real cluster analysis tasks \cite{intro5}. Differing from the numerical attributes with precisely quantified values, categorical attributes are usually with a certain number of qualitative values representing some fuzzy concepts \cite{intro1}, which are embedded in an implicit distance space. Categorical attributes can be further distinguished into two sub-types, i.e., nominal and ordinal attributes, where ordinal attribute values are with additional intrinsic order information \cite{ex8}. The conventional taxonomy of attributes is shown in Figure~\ref{fig:att_type}. As the distance spaces of categorical attributes are implicit, how to combine them with the explicit distance spaces of numerical attributes is a key issue for mixed data clustering \cite{intro10}. {Existing efforts can be roughly categorized into two streams: (i) Encode the qualitative categorical attribute values into numerical ones for clustering; (ii) Define a new hybrid distance measure or metric for clustering \cite{intro6}.}

\begin{figure}[H]	
\centerline{\includegraphics[width=3.0in]{ 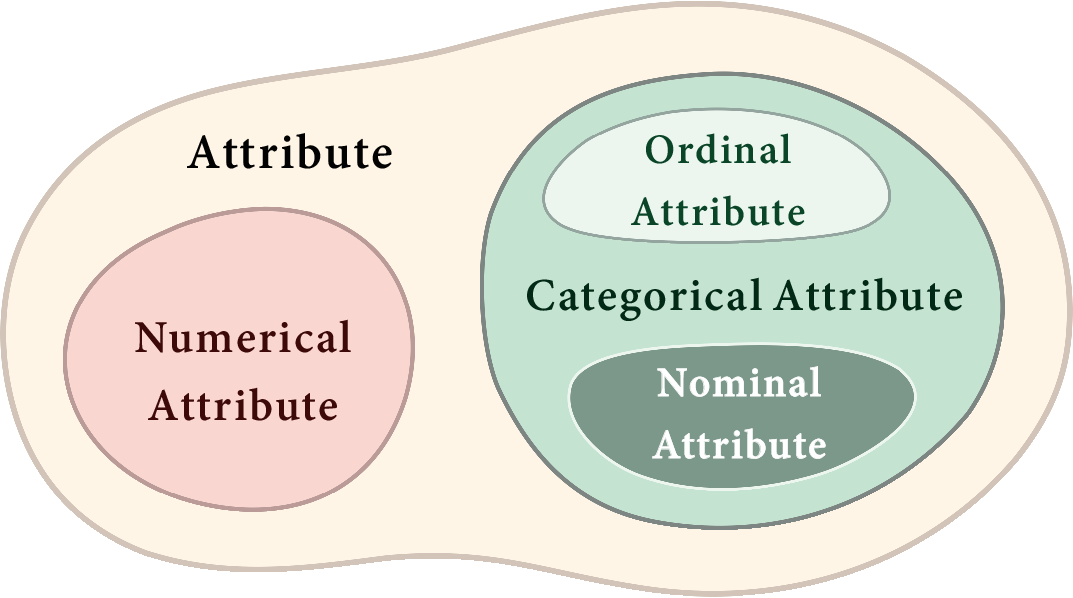}}
\caption{{Conventional taxonomy of attributes, which shows the traditional classification of attributes into Numerical and Categorical types. Categorical Attributes are further divided into Ordinal Attributes, with inherent order (e.g., small, medium, large), and Nominal Attributes, without natural order (e.g., colors or labels). This taxonomy emphasizes the distinct characteristics of each attribute type.} }	
\label{fig:att_type}	
\end{figure}

{For the encoding-based clustering, the popular and simple one-hot encoding converts a possible value of a categorical attribute into a new numerical attribute, with ``1'' representing this possible value and ``0'' representing the others.} In recent years, more informative encoding methods \cite{sbc}\cite{cde} have been proposed for extracting and embedding more statistical information including occurrence frequencies of possible values, co-occurrence frequencies of possible values from dependent attributes, etc., to achieve better representations. Recently, the powerful encoding method \cite{untie} further introduces multiple kernels to comprehensively encode the value-level and attribute-level coupling relationships of categorical attributes. The above-mentioned encoding results are usually fed to the existing clustering algorithms, e.g., the k-means algorithm \cite{kms} and its variants \cite{wkm}\cite{ewkm}\cite{inikms}, for clustering. However, since they are designed for purely categorical data only, the relationship information between numerical and categorical attributes in mixed data will be somewhat omitted by using them.

For the approaches based on dissimilarity defining, Hamming distance metric \cite{hdm} that only simply distinguishes the distance between identical and different categorical values is the most common. To more finely distinguish dissimilarity degrees of different values, some measures \cite{gsm}\cite{lsm}\cite{ebdmjournal} have been proposed based on the occurrence probabilities of possible values. There are also measures \cite{cbdm_journal}\cite{jdm}\cite{cms}\cite{udm} that take the interdependence of categorical attributes into account further. Most recently, more advanced dissimilarity learning methods \cite{dlc}\cite{hd-ndw} have been proposed, which not only reasonably define the dissimilarities, but also enable the adaptive learning of dissimilarities during clustering. As all the above-mentioned solutions are designed for categorical attributes only, the defined dissimilarities are usually combined with Euclidean distances of numerical attributes by the k-prototypes algorithm \cite{kpt} for mixed data clustering. In the literature, some measures \cite{oc}\cite{woc}\cite{Gower} also attempt to directly quantify the object-cluster dissimilarities for mixed data.

Although existing works provide a variety of solutions, they may still suffer from at least one of the following two intractable issues: {(i) Adaptability: As the encoding or dissimilarity defining phase is independent of the clustering phase, the encoded data or the dissimilarities may not suit certain clustering tasks well. How to adapt the representation or the defined dissimilarities of mixed data to the clustering task is a key issue to enhance the scalability and accuracy of the clustering approach; (ii) Homogeneity: Numerical attribute values finely grade the tendencies between two concepts (e.g., \{high, low\} of attribute ``temperature''), while categorical attribute values usually represent multiple concepts (e.g., \{driver, lawyer, nurse, ...\} of attribute ``occupation'') in a coarse-grained manner.} To tackle these issues, a mixed data clustering method \cite{het2hom} has been proposed, which first reconstructs categorical attributes into a form similar to that of numerical attributes, and then uniformly learns the weights of all the attributes. However, it overlooks the potential ordinal relationship of categorical attribute values and only learns the importance of each whole attribute without finely exploring the distance structure among attribute values.


This paper, therefore, studies the differences and connections among numerical, nominal, and ordinal attributes. {For all heterogeneous attributes to provide comparable distances for clustering, we reconstruct each categorical attribute that is originally arranged in a nonlinear space into a set of possible one-dimensional spaces that are homogeneous to the distance space of the numerical attribute. In comparison with the mainstream approaches that rely on external knowledge or hypothesized distributions for categorical attribute distance space representation like Gower's distance~\cite{Gower}, our reconstruction is only based on the basic data statistics to avoid priori bias.} It turns out that a categorical attribute can be adequately represented from different views of its value pairs, and represented attributes are with distance spaces like that of numerical attributes. {Such design provides a homogeneous basis for jointly considering heterogeneous attributes, and we further arrange learnable weights to the distances among reconstructed categorical values to leverage adaptive learning w.r.t. clustering.}
The main contributions of this paper are four-fold:
\begin{itemize}
  \item {Connection among numerical, nominal, and ordinal attributes is revealed from a new perspective of the intrinsic semantic concepts represented by the attribute values, which offers the potential to jointly understand the heterogeneous attributes in mixed datasets.}
  \item {A projection-based attribute reconstruction method is proposed to informatively convert the heterogeneous distance spaces of different types of attributes to be homogeneous, which provides an effective data basis for the processing and analysis of mixed data.}
  \item The adaptation of reconstructed representation to the data objects clustering is designed into a learnable task. {As all types of attributes can be homogeneously and interactively represented by learning, such clustering paradigm is suitable for any-type-attributed data}\footnote{Any-type-attributed data is a dataset composed of any combination of numerical, nominal, and ordinal attributes, which is also called any-type data interchangeably in this paper.}.
  \item Two learning algorithms are instantiated from the learning paradigm, which circumvents the hyper-parameters for learning, and searches clusters in attribute subspaces to improve the degree of learning freedom, respectively. 
\end{itemize}

The rest of this paper is organized as follows. Section~\ref{sct:related} provides an overview of the {related works}. In Section~\ref{sct:proposed}, the proposed representation method and the learning algorithms are presented. Section~\ref{sct:theory} presents theoretical analyses. Section~\ref{sct:experiments} conducts experimental evaluation. Section~\ref{sct:conclusion} draws a conclusion of this paper.

\section{Related Work}
\label{sct:related}

This section provides an overview of the closely related research, i.e., data encoding approaches and dissimilarity measures, for the clustering of categorical and mixed data.

Encoding approaches have two common procedures: (1) encode the target dataset according to a certain strategy, (2) conduct cluster analysis by treating the encoded dataset as numerical-valued data. In this stream, one-hot encoding is the simplest but commonly used approach, which converts the $s$th possible value $o^r_s$ of attribute $a^r$ into a $v^r$-bit vector with the $s$th bit set at ``1'' and the remainders set at ``0''. Here, $v^r$ indicates the total number of possible values of attribute $a^r$. Since such encoding process is equivalent to uniformly assigning distance ``1'' to any pair of different values, one-hot encoding is infeasible to produce different dissimilarity degrees \cite{intro10}\cite{intro6}, and may thus cause information loss. Space structure-based encoding proposed in \cite{sbc} encodes a target data object $\textbf{x}_i$ by concatenating the distances from $\textbf{x}_i$ to all the $n$ data objects into an $n$-dimensional vector. Later, coupling-based representation \cite{cde}\cite{cde_conf} has been proposed considering the relationship among categorical attributes for encoding. Most recently, a more advanced method \cite{untie} adopting different kernels for more comprehensive relationship representation has been proposed. As all the above-mentioned approaches are designed for categorical attributes only, some approaches \cite{mai}\cite{mix2vec} have also been designed to simultaneously encode the numerical and categorical attributes with considering their dependency. However, sensitivity to the hyper-parameters may still influence their performance, and the potential order relationship information of categorical attributes remains to be exploited.

For the approaches that directly define the dissimilarities, Hamming distance is the most popular metric. It always assigns distance ``1'' to two different values and set distance at ``0'' for identical values. Similar to the one-hot encoding, the boolean distance values prevent it from distinguishing different dissimilarity degrees. Accordingly, similarity measures based on probability \cite{gsm}\cite{oc}\cite{ML23KBS1} have been proposed, which can more finely define similarities according to the occurrence frequencies of the possible values. Entropy-based similarity measures \cite{lsm}\cite{ebdmjournal}\cite{ebdmconf}\cite{KHAN23KBS5} proposed to consider distance measurement from the aspect of information theory \cite{intro9}. They also compute similarities based on the occurrence frequencies of possible values, which take a common idea that two more dissimilar values may result in a higher information amount reflected by a larger entropy value. Since they all treat attributes independently and ignore the valuable information provided by the inter-attribute dependence \cite{intro10}, some measures {\cite{cbdm_journal}\cite{adm}\cite{abdm}\cite{cbdm_conf}} have been proposed to define similarities between possible values according to their conditional probability distributions obtained from the relevant attributes. Among them, Ahmad's measure \cite{adm} and the association-based measure \cite{abdm} adopt a similar idea that the similarity between two possible values from an attribute can be reflected by their corresponding conditional probability distributions on the other attributes. Context-based measure\cite{cbdm_journal}\cite{cbdm_conf} further filters out the irrelevant attributes. However, since they rely on the sole interdependence of attributes, they may fail in distance measurement when all the attributes are independent of each other \cite{jdm}\cite{YU2022KBS4}.To address such issue, a series of distance measures proposed in {\cite{jdm}\cite{cms}\cite{udm}\cite{woc}\cite{Adc} }consider both the intra- and inter-attribute statistical information. Some advanced clustering algorithms with attributes weighting mechanism \cite{scc}\cite{mwkm} have also been proposed to learn the importance of categorical attributes during clustering.

Since encoding and dissimilarity defining operations are usually performed separately from clustering, the encoding results or the defined dissimilarities cannot adapt well to different clustering tasks. {Accordingly, some recently proposed methods \cite{dlc}\cite{hd-ndw}\cite{COForest} have been proposed to first define the similarities, and then iteratively learn the similarities and data partitions. As a result, this type of method achieves competitive clustering accuracy compared to the existing clustering approaches. However, the method proposed in \cite{dlc} only considers pure ordinal data while the methods presented in \cite{hd-ndw} and \cite{COForest} are only designed for mixed data composed of ordinal and nominal attributes.} Most recently, a new clustering approach \cite{sig_clust} has been proposed to explore significant clusters of categorical data without requiring a known true number of clusters.

\section{Proposed Method}
\label{sct:proposed}

We first provide preliminaries and formulate the problem, then we introduce the attribute representation method and the learning algorithms in detail. {Frequently used symbols in this paper and the corresponding explanations are sorted out in Table~\ref{tb:symbol}.}

\begin{table}[h]

\caption{{Frequently used symbols. Note that we uniformly use lowercase, uppercase, bold lowercase, and bold uppercase to indicate value, set, vector, and matrix, respectively.}}
\label{tb:symbol}
\centering
\resizebox{1\columnwidth}{!}{
\begin{tabular}{c|l|c|l}
\toprule
Symbol & Explanation &  Symbol & Explanation \\
\midrule
$X$, $A$, $O$, $M$, and $C$& Dataset, attribute set, possible value set, cluster prototypes set, and cluster set & 
$\mathbf{x}_i$ and $x_{i}^r$& $i$-th data object and $r$-th value of $\mathbf{x}_i$  \\	
$a^r$, $v^r$  & $r$-th attribute, possible value number of $a^r$&	    
$A_u$ and $A_c$& Numerical and categorical attribute set\\	
$A_n$ and $A_o$& Nominal and ordinal attribute set of categorical attributes, $A_c=A_n \cup A_o$ &
$d$, $k$ and $n$  & Number of attributes( $d=d_u+d_c$), clusters, and data objects\\
$d_c$ and $d_u$ & Number of categorical attributes ($d_c=d_n+d_o$), and  numerical attributes &
$o_h^r$ &  $h$-th possible value of $a^r$ \\
$\mathbf{Q}$ and $c_l$ & Object-cluster affiliation matrix and $l$-th cluster  &
$q_{il}$ & A value indicating the affiliation between $\textbf{x}_i$ and $c_l$\\
$\textbf{m}_l$ & A vector describing data objects of $c_l$ &
$\gamma^r$ & Number of endogenous spaces corresponding to $a^r$ \\
$\mathcal{R}^r$ and $\mathcal{R}^r_b$ & Endogenous space set of $a^r$ and its $b$-th endogenous space&
 
$\Phi(\cdot,\cdot)$  and $\phi(\cdot,\cdot)$& Data object-level dissimilarity and value-level distance\\
$\mathbf{W}$ and $w^r$& Attribute weight set and and weight indicating the importance of $a^r$  &
$\kappa(\cdot,\cdot)$ & Base distance \\
\bottomrule
\end{tabular} }
\end{table}

\subsection{Preliminaries}
\label{subsct:problem}

Given a mixed dataset $X$ containing $n$ data objects $\{\textbf{x}_1,\textbf{x}_2,...,\textbf{x}_n\}$. Each data object $\textbf{x}_i=[x^1_i,x^2_i,...,x^d_i]^\top$ can be viewed as a $d$-dimensional vector with values from $d$ attributes $A=\{a^1,a^2,...,a^d\}$, including $d_u$ numerical attributes and $d_c$ categorical attributes. For simplicity, assuming that the former $d_u$ attributes in $A$ are numerical ones denoted as an attribute set $A_u$, and the latter $d_c$ are categorical attributes denoted as $A_c$. Thus we have $A_u=\{a^1,a^2,...,a^{d_u}\}$ and $A_c=\{a^{d_u+1},a^{d_u+2},...,a^d\}$. Note that the superscript denotes the index of attribute throughout the paper. Accordingly, we have $A=A_u\cup A_c$, and $d=d_u+d_c$. The $v^r$ possible values of each categorical attribute $a^r\in A_c$ are represented as a set $O^r=\{o^r_1,o^r_2,...,o^r_{v^r}\}$.

Since categorical attributes may also include ordinal attributes, we assume the former $d_n$ attributes in $A_c$ are nominal, and the latter $d_o$ are ordinal. Accordingly, we have $A_c=A_n\cup A_o$ and $d_c=d_n+d_o$. The sequential numbers $\{1,2,...,v^r\}$ further reflect the rank of possible values $O^r=\{o^r_1,o^r_2,...,o^r_{v^r}\}$ of an ordinal attribute $a^r$. For dataset without ordinal attribute, it is possible that $d_o=0$ and $d_c=d_n$.

A general goal of clustering is to partition the $n$ data objects into $k$ compact clusters $C=\{c_1,c_2,...,c_k\}$, where the intra-cluster objects are similar and inter-cluster objects are distinct. An $n\times k$ matrix $\textbf{Q}$ is usually maintained during clustering indicating the affiliations between objects and clusters. More specifically, the value of the $(i,l)$th entry of $\textbf{Q}$ can be obtained by
\begin{equation}\label{eq:obj2}
q_{il}=\left\{
\begin{array}{lll}
1&,   & \text{if}\ l=\arg\min\limits_y\Phi(\textbf{x}_i,\textbf{m}_y)\\
0&, & \text{otherwise}.\\
\end{array}
\right.
\end{equation}
As we focus on the crisp partitional clustering, $q_{il}$ should satisfy $\sum_{l=1}^kq_{il}=1$ and $q_{il}\in\{0,1\}$. $\Phi(\textbf{x}_i,\textbf{m}_l)$ denotes the dissimilarity between object $\textbf{x}_i$ and cluster $c_l$. We describe a cluster $c_l$ using a vector $\textbf{m}_l=[m^1_l,m^2_l,...,m^d_l]^\top$ called prototype, which is computed following the conventional k-prototypes algorithm \cite{kpt}:
\begin{equation}\label{eq:obj3}
m^r_l=\left\{
\begin{array}{lll}
\frac{1}{\sigma(c_l)}\sum_{i=1}^nq_{il}x_i^r&,  & \text{if}\ r\in\{1,2,...,d_u\}\\
o^r_h&,& \text{if}\ r\in\{d_u+1,d_u+2,...,d\},\\
\end{array}
\right.
\end{equation}
where $\sigma(\cdot)$ counts the number of elements in a set, and thus $\sigma(c_l)$ counts the number of objects assigned to $c_l$. $o^r_h$ is obtained by
\begin{equation}\label{eq:obj4}
h=\arg\max_u\sum_{i=1}^n q_{il}\delta(x^r_i=o^r_u),
\end{equation}
where $\delta(x^r_i=o^r_u)$ is a logical function, which returns value ``1'' if $x^r_i=o^r_u$ is true, otherwise, returns ``0''. The object-cluster dissimilarity $\Phi(\textbf{x}_i,\textbf{m}_l)$ in Eq.~(\ref{eq:obj2}) can be written as
\begin{equation}\label{eq:dissim1}
\Phi(\textbf{x}_i,\textbf{m}_l)=\sum_{r=1}^d\phi(x^r_i,m^r_l),
\end{equation}
where the inner-sum $\phi(x^r_i,m^r_l)$ is the dissimilarity between the $r$th values of $\textbf{x}_i$ and $\textbf{m}_l$.

{Since numerical, nominal, and ordinal attributes are with different distance space structures, they may provide heterogeneous information in terms of clustering. Thus, how to homogeneously represent them to eliminate their awkward information gap is the key problem that should be solved for mixed data clustering.}

\subsection{Homogeneous Attribute Representation}
\label{subsct:pbr}


{To eliminate the awkward information gap between heterogeneous attributes, we need to represent them homogeneously.} From the perspective of concepts described by the values of an attribute, the values of a numerical attribute describe tendencies between two (usually opposite) concepts (e.g., \{high, low\} of attribute ``temperature'') while the values of a categorical attribute usually describe multiple concepts (e.g., \{driver, lawyer, nurse, police\} of an attribute ``occupation''). The ordinal attributes are also common among categorical attributes, where multiple ordered concepts (e.g., \{accept, weak accept, marginal, reject\}) describe tendencies between two concepts (e.g., \{accept, reject\} of attribute ``recommendation''). Figure~\ref{fig:concept} shows the concepts of the three types of attributes.

More specifically, the values of a numerical attribute $a^r\in A_u$ are linearly arranged between two conceptual values $o^r_{min}$ and $o^r_{max}$ in a one-dimensional space, where $o^r_{min}=0$ and $o^r_{max}=1$ if the values are normalized into the interval $[0,1]$. By contrast, the values of a nominal attribute $a^r\in A_c$ can be viewed as concentrating on $v^r$ non-linearly arranged conceptual values in $O^r$. The case between numerical and nominal attributes is the ordinal attribute, where the attribute values concentrate on $v^r$ linearly arranged conceptual values in $O^r$. Numerical attribute is with well-defined Euclidean distance space, which features easy-to-represent and computational convenience. {Therefore, how to represent a categorical attribute in one-dimensional space like numerical attribute but still preserve the original information of the attribute is the goal of our representation.} 

\begin{figure}[h]	
\centerline{\includegraphics[width=3.6in]{ 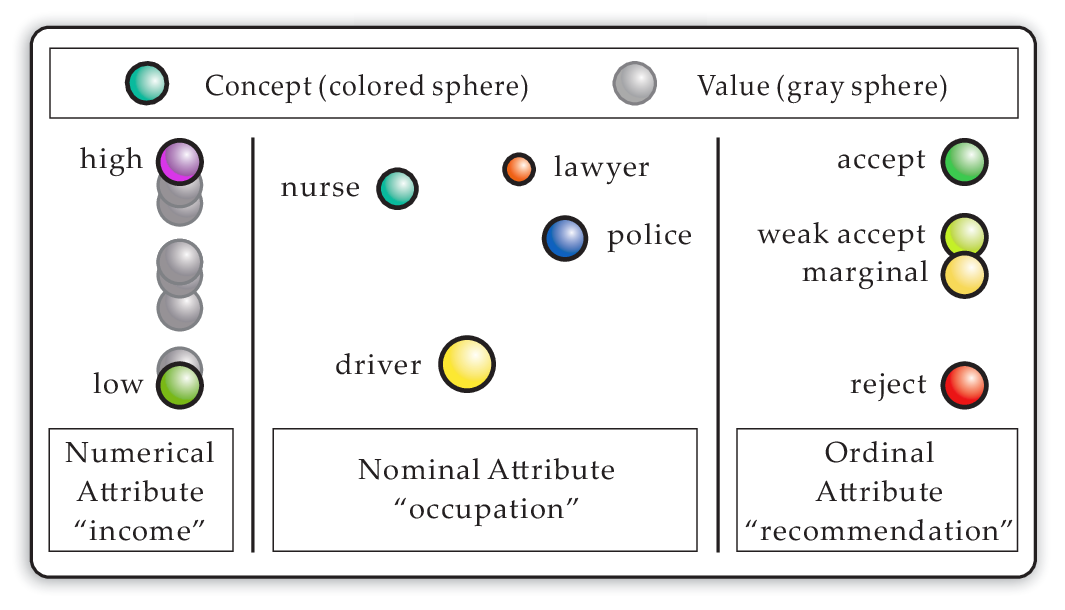}}
\caption{{Comparison of the ``concepts'' of numerical attribute, nominal attribute, and ordinal attribute, which illustrates how ``concepts" and ``values" are represented across different attribute types. Numerical attributes, like ``income," maintain a continuous order (e.g., high to low). Nominal attributes, such as ``occupation," consist of distinct categories without inherent order (e.g., nurse, lawyer, driver). Ordinal attributes, like ``recommendation," contain ordered categories (e.g., accept, weak accept, marginal, reject), highlighting their hierarchical structure.}}	
\label{fig:concept}	
\end{figure}

{The information carried by categorical attributes refers to their multidimensional distance structure, which is rich in information.} However, when converted into a one-dimensional Euclidean distance structure, this can easily lead to a loss of the original information. To preserve this information, we use a  multiple spaces projection approach, which retains the rich structure while enabling the conversion into a one-dimensional distance structure similar to numerical attributes. This allows categorical attributes to participate seamlessly in arithmetic operations during clustering. However, when One-Hot Encoding (OHE) is used, it imposes an equidistant constraint on the distance structure, preventing the original information of categorical attributes from being fully preserved.

For a categorical attribute $a^r\in A_c$, distance between any pair of values $o^r_g$ and $o^r_h$ can be informatively measured based on dataset statistics by
\begin{equation}\label{eq:cpddiff}
\kappa(o^r_g,o^r_h)=\sum_{s=1}^{d}\sum_{j=1}^{v^s}|p(o_j^s|o_g^r)-p(o_j^s|o_h^r)|{,}
\end{equation}
where the term $p(o_j^s|o_g^r)$ is actually the occurrence probability of the value $o^s_j$ as given $o^r_g$, which can be written as
\begin{equation}\label{eq:cp}
p(o_j^s|o_g^r)=\frac{\sum_{u=1}^n\delta(x^s_u=o^s_j\wedge x^r_u=o^r_g)}{\sum_{i=1}^n\delta(x^r_i=o^r_g)}.
\end{equation}
Eq.~(\ref{eq:cpddiff}) measures the overall difference between two Conditional Probability Distributions (CPDs) obtained from attribute $a^s$ as given two possible values $o^r_g$ and $o^r_h$ from $a^r$. Distance defined in this form has been commonly adopted by the recent works \cite{cms}\cite{hd-ndw} proposed for categorical data, and is generalized to HOmogeneous Distance (HOD) \cite{hd-ndw} under the case that the dataset is composed of both nominal and ordinal attributes. To further exploit the information offered by numerical attributes, we discrete numerical attributes and treat them as ordinal attributes to reflect distances of categorical attributes. This distance is called base distance as the following projection-based representation will be implemented based on these distances. {Figure~\ref{fig:dist_indication} shows the reflection relationship between different types of attributes for computing $\kappa(o^r_g,o^r_h)$.}

\begin{figure}[h]	
 
\centerline{\includegraphics[width=3.6in]{ 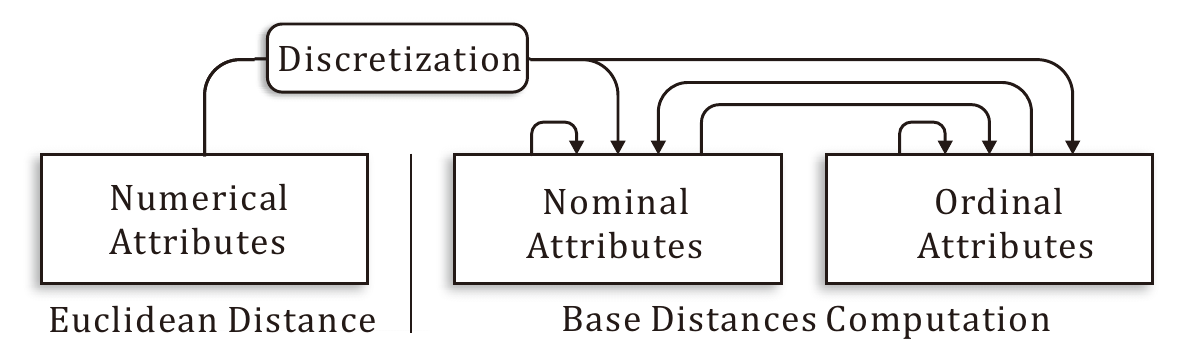}}
\caption{Relationships among different types of attributes in base distances computation. ``A$\rightarrow$B'' means that base distances of Type-B attributes are computed with the contribution of Type-A attributes.}	\label{fig:dist_indication}	
\end{figure}

{However, distance spaces of categorical attributes yielded by HOD are not guaranteed to be one-dimensional spaces like that of numerical attributes, see Figure~\ref{fig:concept}.} To represent the heterogeneous space structures of numerical and categorical attributes in a homogeneous way, we dismantle the space structure of a categorical attribute $a^r$ by projecting all its $n$ attribute values onto each of the one-dimensional spaces $\mathcal{R}^r=\{\mathcal{R}^{r,1},\mathcal{R}^{r,2},...,\mathcal{R}^{r,\gamma^r}\}$ spanned by $\gamma^r$ pairs of conceptual values (i.e. possible values) where $\gamma^r=v^r(v^r-1)/2$. The projection is implemented by geometrically projecting the $n$ values onto each space based on the base distance $\kappa(o^r_g,o^r_h)$. {Specifically, projection point of a value $o^r_t$ on a space $\mathcal{R}^{r,b}$ ($b\in\{1,2,...,\gamma^r\}$) spanned by $o^r_g$ and $o^r_h$ can be described by the distance between $o^r_t$ and $o^r_g$, and also interchangeably by the distance between $o^r_t$ and $o^r_h$}, which can be computed by:
\begin{equation}\label{eq:pp1}
\phi(o^r_t,o^r_g;\mathcal{R}^{r,b})\!=\!\frac{|\kappa(o^r_t,o^r_g)^2\!-\!\kappa(o^r_t,o^r_h)^2\!+\!\kappa(o^r_g,o^r_h)^2|}{2\kappa(o^r_g,o^r_h)}
\end{equation}
and
\begin{equation}\label{eq:pp2}
\phi(o^r_t,o^r_h;\mathcal{R}^{r,b})\!=\!\frac{|\kappa(o^r_t,o^r_h)^2\!-\!\kappa(o^r_t,o^r_g)^2\!+\!\kappa(o^r_g,o^r_h)^2|}{2\kappa(o^r_g,o^r_h)},
\end{equation}
respectively. {Eqs}.~(\ref{eq:pp1}) and~(\ref{eq:pp2}) are derived by directly applying the Pythagorean theorem. A more intuitive demonstration of such a process of a categorical attribute $a^r$ is shown in Figure~\ref{fig:principle}. An attribute $a^r$ with $4$ possible values is expanded into $6$ distinct attributes after the projection process. Each projected attribute represents unique information, highlighting the rich information content of the original attribute, which naturally justifies the expansion factor. This scaling is determined by the number of possible values for the categorical attribute. {Therefore, dividing its contribution equally among the $6$ projected attributes would dilute its inherent richness and is not appropriate.}

\begin{figure}[h]	
 
\centerline{\includegraphics[width=4.6in]{ 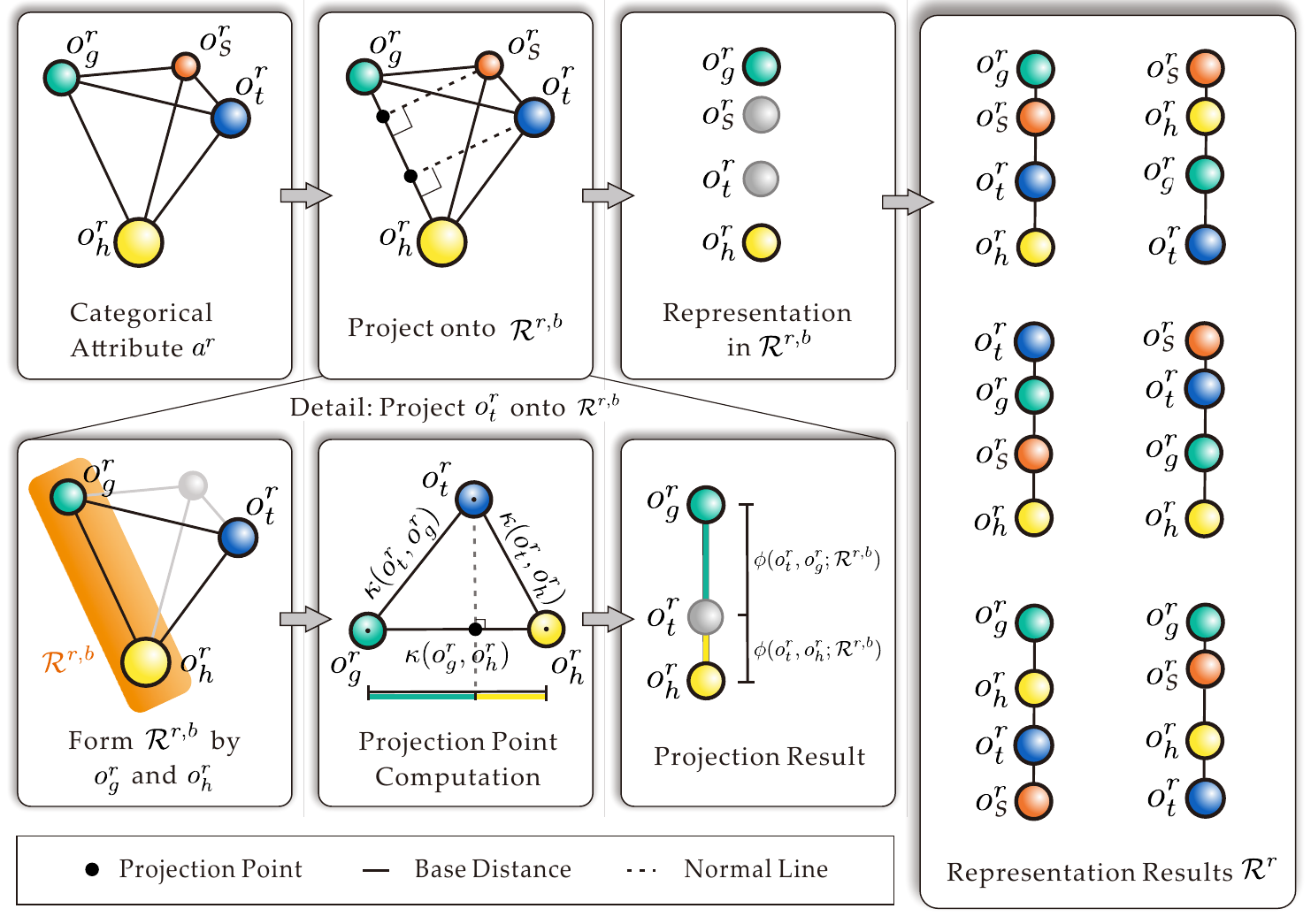}}  
\caption{{Projection processes of a categorical attribute $a^r$. An attribute $a^r$ with $4$ possible values is expanded into $6$ distinct attributes after the projection process. Each projected attribute represents unique information, highlighting the rich information content of the original attribute.} }	
\label{fig:principle}
\end{figure}

{The projection for ordinal attributes is a special case for the projection of categorical attributes.} Since we assume that the ordinal values are arranged linearly as shown in Figure~\ref{fig:concept}, the $\gamma^r$ one-dimensional spaces spanned by the conceptual value pairs are exactly overlapped, and so do the corresponding projection results. Thus, only one one-dimensional space $\mathcal{R}^r=\{\mathcal{R}^{r,1}\}$ where $\gamma^r=1$ will be enough for representing an ordinal attribute $a^r$, and the projection point of a value $o^r_t$ can be directly reflected by the based distances:
\begin{equation}\label{eq:pp3}
\phi(o^r_t,o^r_g;\mathcal{R}^{r,b})=\kappa(o^r_t,o^r_g)
\end{equation}
and
\begin{equation}\label{eq:pp4}
\phi(o^r_t,o^r_h;\mathcal{R}^{r,b})=\kappa(o^r_t,o^r_h)
\end{equation}
without loss of the generality for both the two types of categorical attributes.

{Given a certain space $\mathcal{R}^{r,b}$ spanned by $o^r_g$ and $o^r_h$, all the $n$ values of $a^r$ are linearly arranged after the projection, and the distance between any two values $o^r_u$ and $o^r_f$ represented in $\mathcal{R}^{r,b}$ can be written based on Eqs.~(\ref{eq:pp1})-(\ref{eq:pp4})} as
\begin{flalign}\label{eq:ppdist1}
\phi(o^r_u,o^r_f;\mathcal{R}^{r,b})&=|\phi(o^r_u,o^r_g;\mathcal{R}^{r,b})-\phi(o^r_f,o^r_g;\mathcal{R}^{r,b})|\nonumber\\
&=|\phi(o^r_u,o^r_h;\mathcal{R}^{r,b})-\phi(o^r_f,o^r_h;\mathcal{R}^{r,b})|.
\end{flalign}
The projection of $a^r$ actually generates $\gamma^r$ sub-attributes $A^r=\{a^r_1,a^r_2,...,a^r_{\gamma^r}\}$ with common possible values $O^r$ represented in $\gamma^r$ spaces $\mathcal{R}^r$. Note that $A^r$ with superscript $r$ indicates sub-attributes obtained after the representation of original attribute $a^r$, while $A_c$ with subscript $c$ indicates that the attribute type is categorical. After representing all the $d_c$ categorical attributes, attribute set $A$ is updated to
\begin{equation}\label{eq:att_all}
\hat{A}=A_u\cup\hat{A}_c,
\end{equation}
where
\begin{equation}\label{eq:att_cate}
\hat{A}_c=A^{d_u+1}\cup A^{d_u+2}\cup...\cup A^d.
\end{equation}
The symbol ``hat'' denotes that an original set or quantity is updated after the representation, see $\hat{A}$ and $\hat{A}_c$. {Accordingly, we have $\hat{d}=d_u+\hat{d}_c$, where $\hat{d}_c=\sum_{r=d_u+1}^d\gamma^r$.}

\begin{figure*}[h]	
\centerline{\includegraphics[width=6.6in]{ 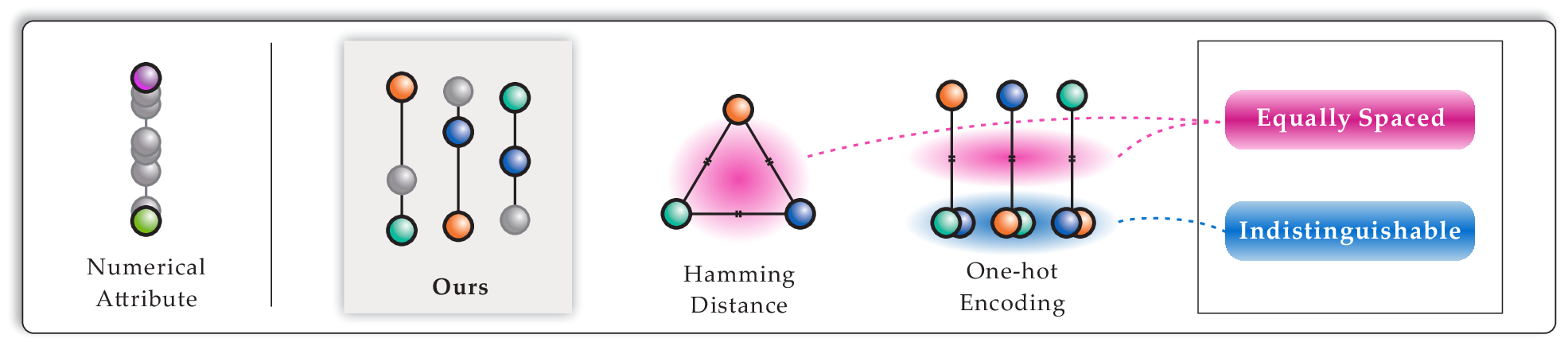}}
\caption{{Comparison of Hamming distance, one-hot encoding, and categorical attribute representation obtained through our method. It is obvious that our representation is more informative, and provides a homogeneous basis for distance computation on heterogeneous attributes.}}	
\label{fig:compare}	
\end{figure*}

{Based on the updated attribute set $\hat{A}$, we divide all distances of each $a^r\in\hat{A}_c$ by the maximum distance of $a^r$ to make the distances comparable to the distances of normalized numerical attributes in $A_u$.} The obtained distances $\phi(o^r_u,o^r_f;\mathcal{R}^{r,b})$ can be interchangeably denoted as $\phi(o^r_u,o^r_f)$ for simplicity without causing ambiguity. $\phi(o^r_u,o^r_f)$ calculates the absolute value of the dissimilarity between the encoded coordinates of $o^r_u$ and $o^r_f$ (see the ``Projection Result'' part of Figure~\ref{fig:principle}), which is equivalent to calculating the absolute value of the distance between two values on a numerical attribute. Because the differences among variables (attributes) in mixed data is often large, we use the Manhattan distance for Eq.~(\ref{eq:dissim1}). 

An intuitive comparison of attributes represented by the proposed method and the conventional methods is shown in Figure~\ref{fig:compare}. The representation of Hamming distance is significantly different from the structure of the numerical attribute, and all the different values are with identical distances. One-hot encoding has the additional problem of making different values indistinguishable. It is obvious that our representation is more informative, and provides a homogeneous basis for distance computation on heterogeneous attributes. {To obtain more clustering-friendly heterogeneous attribute representations, the key lies in how to integrate the representation process with clustering through joint learning.}

\subsection{Learning Algorithms}
\label{subsct:ls}

{To facilitate the learning of representations in clustering}, the distance $\Phi(\textbf{x}_i,\textbf{m}_l)$ in Eq.~(\ref{eq:dissim1}) is re-written as the linear combination of the dissimilarities between $x^r_i$ and $m^r_l$ reflected by different attributes in $\hat{A}$:
\begin{equation}\label{eq:dissim2}
\Phi_w(\textbf{x}_i,\textbf{m}_l)=\sum_{r=1}^{\hat{d}}\phi(x^r_i,m^r_l)\cdot w^r,
\end{equation}
where $w^r$ controls the contribution of $a^r$, $\textbf{w}=[w^1,w^2,...,w^{\hat{d}}]^\top$, and $\sum_{r=1}^{\hat{d}}w^r=1$. To learn the representations while adapting to the clustering task, data objects partition $\textbf{Q}$, cluster prototypes $M=\{\textbf{m}_1,\textbf{m}_2,...,\textbf{m}_k\}$, and the weights of attributes $\textbf{w}$ are iteratively updated in the following processes: (1) Let $\textbf{w}$ and $M$ be fixed, compute $\textbf{Q}$; (2) Let $\textbf{w}$ and $\textbf{Q}$ be fixed, compute $M$; (3) Iteratively implement (1) and (2) until convergence, then fix $\textbf{Q}$ and $M$, and update $\textbf{w}$. The above process follows the common learning steps for clustering with attributes weighting \cite{wkm}, and the updating strategy of $\textbf{Q}$ and $M$ have been presented by Eqs.~(\ref{eq:obj2}) -~(\ref{eq:obj4}) in Section~\ref{subsct:problem}.

The main difficulty lies in the updating of $\textbf{w}$. As mentioned in Section~\ref{subsct:pbr}, we represent an original attribute into a set of same-source sub-attributes $A^t$ with common possible values $O^t$ described by different distance structures computed by Eqs.~(\ref{eq:pp1})-(\ref{eq:ppdist1}), the attributes in $A^t$ can thus be considered to be more dependent. If we directly compute $\textbf{w}$ using the Lagrangian multiplier method or the gradient-decent method in the step (3), the weights of the same-source sub-attributes with similar represented distance structures will over-enhance each other in the next step (1). More specifically, suppose $\hat{a}^r$ and $\hat{a}^s$ are the same-source attributes in $A^t$, if the total intra-cluster distance contributed by $a^r$ and $a^s$, i.e., $D^r=\sum_{i=1}^n\sum_{l=1}^kq_{il}\phi(x^r_i,m^r_l)$ and $D^s=\sum_{i=1}^n\sum_{l=1}^kq_{il}\phi(x^s_i,m^s_l)$, are both relatively low and are close to each other, then the effects of the weights $w^r$ and $w^s$ computed according to $D^r$ and $D^s$, respectively, will over-enhance each other in the next step (1) because they may have similar effects in forming clusters. Such a synergy effect can easily cause a corrupt solution.

{Accordingly, we adopt a novel weight updating strategy that more comprehensively considers the intra-cluster compactness and inter-cluster separation contributed by attributes to obtain discriminative quantification of the attribute weights.} Specifically, a weight $w^r$ is computed by
\begin{equation}\label{eq:update1}
w^r=\frac{I^r}{\sum_{s=1}^{\hat{d}}I^s},
\end{equation}
where $I^r$ is defined as
\begin{equation}\label{eq:update2}
I^r=\frac{S^r}{D^r}.
\end{equation}
We adopt average intra-cluster distance
\begin{equation}\label{eq:update3}
D^r=\frac{1}{n}\sum_{i=1}^n\sum_{l=1}^kq_{il}\phi(x^r_i,m^r_l)
\end{equation}
to {indicate the importance of $a^r$ in terms of the intra-cluster compactness, and adopt average inter-cluster distance}
\begin{equation}\label{eq:update4}
S^r=\frac{1}{n(k-1)}\sum_{i=1}^n\sum_{l=1}^k(1-q_{il})\phi(x^r_i,m^r_l)
\end{equation}
to {quantify the importance of $a^r$ in terms of the inter-cluster separation.} We compute average intra- and inter-cluster distance to make $D^r$ and $S^r$ comparable, and a more important attribute $a^r$ should have a smaller $D^r$ and a larger $S^r$. Moreover, $\frac{1}{n}$ in Eq.~(\ref{eq:update3}) and $\frac{1}{n(k-1)}$ in Eq.~(\ref{eq:update4}) can be simultaneously omitted, as their division $\frac{1}{n(k-1)}/\frac{1}{n}$ is a constant in Eq.~(\ref{eq:update2}). {The proposed learning algorithm is summarized in Algorithm~\ref{alg:hlc}.}

\begin{algorithm}[h]
\caption{\textbf{H}eterogeneous \textbf{A}ttribute \textbf{R}econstruction and \textbf{R}epresentation learning (HARR)} 
\label{alg:hlc}
\begin{description}
  \item[Input:] Dataset $X$, number of clusters $k$.
  \item[Output:] Data partition $\textbf{Q}$ and attribute weights $\textbf{w}$.
  \item[Step 1:] Represent all the categorical attributes according to Eqs.~(\ref{eq:pp1}) -~(\ref{eq:ppdist1}), obtain $\hat{A}$;
  \item[Step 2:] Initialize $\textbf{Q}'$ and $\textbf{Q}''$ by setting all their values to 0 (for judging the convergence); Initialize $M$ with randomly selected $k$ objects; Initialize $\textbf{w}$ by $w^r=1/\hat{d}$;
  \item[Step 3:] Fix $\textbf{w}$ and $M$, compute $\textbf{Q}$ by Eq.~(\ref{eq:obj2}). If $\textbf{Q}\neq\textbf{Q}'$, set $\textbf{Q}'=\textbf{Q}$, and turn to Step 4; Otherwise, turn to Step 5;
  \item[Step 4:] Fix $\textbf{w}$ and $\textbf{Q}$, compute $M$ by Eq.~(\ref{eq:obj3}), go to Step 3;
  \item[Step 5:] If $\textbf{Q}\neq\textbf{Q}''$, set $\textbf{Q}''=\textbf{Q}$, fix $\textbf{Q}$ and $M$, compute $\textbf{w}$ according to Eqs.~(\ref{eq:update1})-(\ref{eq:update4}), and turn to Step 3; Otherwise, stop and output $\textbf{Q}$ and $\textbf{w}$.
\end{description}
\end{algorithm}


{We focus on the learning of linear combinations of represented attributes.} It would waste useful learning evidence if we directly use the contribution $I^r$ of each attribute $a^r$ on all the clusters measured by Eq.~(\ref{eq:update2}), but ignore the specific coupling relation of each attribute-cluster pair. Therefore, we further propose a more advanced learning strategy to adequately exploit the learning evidence provided by the distances caused by the attributes on different clusters. We expand the original weight vector $\textbf{w}$ into a $k\times\hat{d}$ matrix $\mathbf{W}$ with its $(l,r)$th entry stores the weight $w^r_l$ representing the importance of $a^r$ in forming cluster $c_l$. {Following the updating strategy of $\textbf{w}$ described by Eqs.~(\ref{eq:update1})-~(\ref{eq:update4}), we quantify the importance of $w^r_l$} by
\begin{equation}\label{eq:update1m}
w^r_l=\frac{I^r_l}{\sum_{s=1}^{\hat{d}}I^s_l}{,}
\end{equation}
where $I^r_l$ quantifies the contribution of $a^r$ in making the objects of cluster $c_l$ more compact, and more distinct from the objects of the other clusters, which is defined as
\begin{equation}\label{eq:update2m}
I^r_l=\frac{S^r_l}{D^r_l}.
\end{equation}
We adopt average intra-cluster distance on cluster $c_l$
\begin{equation}\label{eq:update3m}
D^r_l=\frac{1}{\sigma(c_l)}\sum_{i=1}^nq_{il}\phi(x^r_i,m^r_l)
\end{equation}
to {indicate the importance of $a^r$ in compacting the $\sigma(c_l)$ objects of $c_l$.} We adopt average distance formed between $c_l$ and objects outside $c_l$
\begin{equation}\label{eq:update4m}
S^r_l=\frac{1}{n-\sigma(c_l)}\sum_{i=1}^n(1-q_{il})\phi(x^r_i,m^r_l)
\end{equation}
to {quantify the importance of $a^r$ in separating the $\sigma(c_l)$ objects in $c_l$ from the rest $n-\sigma(c_l)$ objects.} The corresponding algorithm is summarized as Algorithm~\ref{alg:hlcm}. {As Algorithm~\ref{alg:hlc} updates a weight \textbf{V}ector $\textbf{w}$, while Algorithm~\ref{alg:hlcm} updates a weight \textbf{M}atrix $\textbf{W}$, these two algorithms are denoted as HARR-V and HARR-M, respectively, to distinguish.} 

\begin{algorithm}[h]
\caption{HARR to update weight \textbf{M}atrix (HARR-M)}
\label{alg:hlcm}
\begin{description}
  \item[Input:] Dataset $X$, number of clusters $k$.
  \item[Output:] Data partition indicator matrix $\textbf{Q}$ and attribute weights $\textbf{W}$.
  \item[Step 1:] Represent all the categorical attributes according to Eqs.~(\ref{eq:pp1}) -~(\ref{eq:ppdist1}), obtain $\hat{A}$;
  \item[Step 2:] Initialize $\textbf{Q}'$ and $\textbf{Q}''$ by setting all their values to 0 (for judging the convergence); Initialize $M$ with randomly selected $k$ objects; Initialize $\textbf{W}$ by $w^r_l=1/\hat{d}$;
  \item[Step 3:] Fix $\textbf{W}$ and $M$, compute $\textbf{Q}$ by Eq.~(\ref{eq:obj2}). If $\textbf{Q}\neq\textbf{Q}'$, set $\textbf{Q}'=\textbf{Q}$, and turn to Step 4; Otherwise, turn to Step 5;
  \item[Step 4:] Fix $\textbf{W}$ and $\textbf{Q}$, compute $M$ by Eq.~(\ref{eq:obj3}), go to Step 3;
  \item[Step 5:] If $\textbf{Q}\neq\textbf{Q}''$, set $\textbf{Q}''=\textbf{Q}$, fix $\textbf{Q}$ and $M$, compute $\textbf{W}$ according to Eqs.~(\ref{eq:update1m})-(\ref{eq:update4m}), and turn to Step 3; Otherwise, stop the algorithm, output $\textbf{Q}$ and $\textbf{W}$.
\end{description}
\end{algorithm}

\section{Theoretical Analysis}
\label{sct:theory}


\begin{theorem}\label{the:metric}
Distance measure yielded by the proposed projection-based representation is a {distance metric.}
\end{theorem}

\begin{proof}
Since all the values of a categorical attribute are projected onto one-dimensional Euclidean spaces, the value-level distances computed by Eq.~(\ref{eq:ppdist1}) satisfy:
\begin{enumerate}
\item $\phi(o^r_g,o^r_f)\geq0$
\item $o^r_g=o^r_f\Leftrightarrow \phi(o^r_g,o^r_f)=0 $
\item $\phi(o^r_g,o^r_f)=\phi(o^r_f,o^r_g)$
\item $\phi(o^r_g,o^r_f)\leq\phi(o^r_g,o^r_t)+\phi(o^r_t,o^r_f)$
\end{enumerate}
for any $g,f,t\in\{1,2,...,v^r\}$ and $r\in\{d_u+1,d_u+2,...,\hat{d}\}$. The object-level distance can be written based on the value-level distance as $\Phi(\textbf{x}_a,\textbf{x}_b)=
\sum_{r=1}^{\hat{d}}\phi(x^r_a,x^r_b)$ according to Eq.~(\ref{eq:dissim1}), which can be generalized to $\Phi_w(\textbf{x}_a,\textbf{x}_b)=
\sum_{r=1}^{\hat{d}}\phi(x^r_a,x^r_b)w^r$ where $w^r$ ($w^r\geq 0$) is the weight corresponding to the $r$th represented attribute. According to the metric properties of value-level distances, it is clear that the object-level distances satisfy:
\begin{enumerate}
\item $\Phi_w(\textbf{x}_a,\textbf{x}_b)\geq0$
\item $\textbf{x}_a=\textbf{x}_b\Leftrightarrow \Phi_w(\textbf{x}_a,\textbf{x}_b)=0 $
\item $\Phi_w(\textbf{x}_a,\textbf{x}_b)=\Phi_w(\textbf{x}_b,\textbf{x}_a)$
\item $\Phi_w(\textbf{x}_a,\textbf{x}_b)\leq \Phi_w(\textbf{x}_a,\textbf{x}_c)+\Phi_w(\textbf{x}_c,\textbf{x}_b)$
\end{enumerate}
for any $a,b,c\in\{1,2,...,n\}$ and $r\in\{1,2,...,\hat{d}\}$. Accordingly, the distance measure yielded by the proposed projection-based representation is a distance metric.
\end{proof}

\subsection{Degree of Learning Freedom}

Optimal representation of categorical attributes can be obtained by encode each possible value of the attributes with an optimal numerical value. To search for such optimal representation, each possible value is treated as an encoding variable that will be learned during clustering. For the convenience of analysis without loss of generality, we assume each categorical attribute is with the identical number of possible values $\upsilon=\sum_{r=1}^{d_c}v^r/d_c$.

\begin{lemma} Given $X$ composed of $d_c$ categorical attributes $A_c$, the Maximum number of Possible encoding Variables (MoPV) w.r.t. the clustering task that searches for $k$ clusters is $kd_c(\upsilon-1)$, as the Euclidean space is a one-dimensional space for each attribute, and determining the $\upsilon-1$ distances between each pair of adjacent possible values is equivalent to the determining of the encoded values of all the $\upsilon$ possible values.
\end{lemma}

A higher Degree of Learning Freedom (DoLF) of the independent representation variables can provide a more potential/promising basis for obtaining the optimal attribute representations w.r.t. a clustering task. We analyze the DoLF of the proposed HARR-V/M/T as follows. 


\begin{lemma} Given $X$ composed of $d_c$ categorical attributes $A_c$, DoLF of HARR-V is $d_c\upsilon(\upsilon-1)/2$ as HARR-V represents each of the $d_c$ attribute $a^r$ into each of the $\upsilon(\upsilon-1)/2$ one-dimensional spaces, where each such space is spanned by a possible pair of the $\upsilon$ possible values of $a^r$.
\end{lemma}

\begin{theorem}
HARR-V yields a hyper-DoLF clustering if $\upsilon>2k$.
\end{theorem}
\begin{proof}
Assuming DoLF of HARR-V is larger than MoPV of clustering, then we have $d_c\upsilon(\upsilon-1)/2>kd_c(\upsilon-1)$. By simplification, we have $\upsilon/2>k$, and thus $\upsilon>2k$. 
\end{proof}

\begin{lemma} Given $X$ composed of $d_c$ categorical attributes $A_c$, DoLF of HARR-M is $kd_c\upsilon(\upsilon-1)/2$ as HARR-M represents each attribute $a^r$ into each of the $\upsilon(\upsilon-1)/2$ one-dimensional spaces w.r.t. each of the $k$ clusters.
\end{lemma}

\begin{theorem}
HARR-M yields a hyper-DoLF clustering if $\upsilon>2$. 
\end{theorem}
\begin{proof}
Assuming DoLF of HARR-M is larger than MoPV of clustering, then we have $kd_c\upsilon(\upsilon-1)/2>kd_c(\upsilon-1)$. By simplification, we have $\upsilon/2>1$, and thus $\upsilon>2$.
\end{proof}


\subsection{Time complexity}


It can be seen from Algorithm~\ref{alg:hlc} that under the same $n$, $d$, and $k$, HARR-V has the same time complexity as the conventional clustering algorithm with attributes weighting mechanism\cite{wkm}. {However, due to the representation of attributes, $\hat{d}$ of HARR-V will not be smaller than the original $d$, i.e., $\hat{d}-d=\sum_{r=d_u+1}^dv^r(v^r-1)/2-d_c\geq0$ as discussed in Section~\ref{subsct:pbr}. }But since $v^r$ is a very small integer ranging from 2 to 8 for most categorical attributes as shown in Table~\ref{tb:sta_cate}, $\hat{d}$ will not significantly increase the computation cost of HARR-V. The time complexity of HARR-V is analyzed below.

\begin{theorem}\label{the:timepp}
The time complexity of HARR-V is $O(d^2n+EInkd)$.
\end{theorem}
\begin{proof}
To obtain the worst-case time complexity, we assume that $X$ is a pure nominal data, i.e., $d=d_u$, and let $\nu=max(v^1,v^2,...,v^d)$. Then we separately prove the time complexity of three parts of HARR in the following.

(1) Projection-based representation (i.e., Step 1 of Algorithm~\ref{alg:hlc}): {To compute base distances using Eq.~(\ref{eq:cpddiff}), CPDs from each of the $d$ attributes given each of the $d\nu$ possible values should be obtained first by counting the $n$ attribute values with time complexity $O(d^2\nu n)$. Then, Eq.~(\ref{eq:cpddiff}) (with time complexity $O(d\nu))$ will be triggered $d\nu(\nu-1)/2$ times to compute the base distance between $\nu(\nu-1)/2$ pairs of intra-attribute possible values for each of the $d$ attributes. So the computation of base distances is with time complexity $O(d^2\nu^3)$. By using the prepared base distances, all the $\nu$ possible values of each of the $d$ attributes will be represented onto $\nu(\nu-1)/2$ one-dimensional spaces according to Eqs.~(\ref{eq:pp1}) and~(\ref{eq:pp2}), with time complexity $O(d\nu^3)$.} Therefore, the time complexity of the projection-based representation is $O(d^2\nu n+d^2\nu^3+d\nu^3)$.

(2) Learning of $\textbf{Q}$ and $M$ with fixed $\textbf{w}$ (i.e., Step 3 and 4 of Algorithm~\ref{alg:hlc}): {Let $I$ represent the number of iterations for making Step 3 and 4 converge. In each iteration, the distance between each of the $n$ objects and each of the $k$ clusters reflected by each of the $d\nu(\nu-1)/2$ represented sub-attributes is computed.} Therefore, time complexity for this part is $O(Inkd\nu^2)$.

(3) Computation of $\textbf{w}$ (i.e., Step 5 of Algorithm~\ref{alg:hlc}): {For each $a^r$ from the $d\nu(\nu-1)/2$ represented attributes, $D^r$ and $S^r$ should be computed by summing up $n$ and $n(k-1)$ distances according to Eqs.~(\ref{eq:update3}) and~(\ref{eq:update4}), respectively. As all the value-level distances have been prepared in the projection-based representation phase, the time complexity for computing all the $d\nu(\nu-1)/2$ weights by Eqs.~(\ref{eq:update1}) and~(\ref{eq:update2}) is $O(nkd\nu^2)$.}

{Let $E$ be the number of iterations for making the loop formed by Steps 3-5 converge (i.e., times for repeating the computation of part (2) and part (3) in this proof).} The time complexity of the HARR algorithm is $O(d^2\nu n+d^2\nu^3+d\nu^3+E(Inkd\nu^2+nkd\nu^2))$. Since $\nu$ is a very small constant ($\nu\leq8$ in most cases, see Table~\ref{tb:sta_cate}), the final time complexity of the HARR algorithm is $O(d^2n+EInkd)$.
\end{proof}

Similar to the proof of Theory~\ref{the:timepp}, HARR-M described in Algorithm~\ref{alg:hlc} is with same time complexity as HARR-V.

\section{Experiments}\label{sct:experiments}

We first introduce the experimental settings and then present the evaluation results with analysis.

\subsection{Experimental Settings}

Experimental design, datasets, counterparts to be compared, and validity indices are introduced in the following. All the experiments are coded by MATLAB R2021a. {Reported clustering results are the average performance of 20 runs of the compared approaches.}

\begin{table}[h]

\caption{{Statistics of the 14 real public datasets, including the number of numerical attributes $d_u$, categorical attributes $d_c$, nominal and ordinal categorical attributes $d_n$ and $d_o$, i.e., $d^c=d^n+d^o$, data objects $n$, true number of clusters $k^*$, and the number of possible values of categorical attributes.}}
\label{tb:sta_cate}
\centering
\resizebox{0.9\columnwidth}{!}{
\begin{tabular}{c|cc|cc|cc|cc|c}
\toprule
No. & Dataset & Abbrev.&$d_u$ & $d_c$ &$d_n$ &$d_o$& $n$ & $k^*$ & No. of Possible Values of Categorical Attributes \\
\midrule
1 & Inflammations Diagnosis & DS &1& 5& 5& \ 0 & \ 120 & \ 2 & \{2,\ 2,\ 2,\ 2,\ 2\} \\	
2 & Heart Failure & HF & 7 & 5& 5& \ 0 & \ 299 & \ 2 & \{2,\ 2,\ 2,\ 2,\ 2\} \\	
3 & Autism-Adolescent & AA & 2  & 7& 7& \ 0 & \ 104 & \ 2& \{2,\ 9,\ 2,\ 2,\ 33,\ 2,\ 6\} \\	
4 & Amphibians & AP & 2 & 12& 6& \ 6 & \ 189 & \ 2 &\{8,\ 5,\ 8,\ 7,\ 8,\ 3,\ 5,\ 6,\ 6,\ 6,\ 3,\ 2\}  \\	
5 & Dermatology  & DT & 1& 33& 1&  32 &  \ 358 & \ 6 &\{2,\ 4,\ 4,\ 4,\ 4,\ 4,\ 4,\ 4,\ 4,\ 4,\ 4,\ 4,\ 
 3,\ 4,\ 4,\ 4,\ 4,  \\	
  &  & & &  &&  & && \ 4,\ 4,\ 4,\ 4,\ 4,\ 4,\ 4,\ 4,\ 4,\ 4,\ 4,\ 4,\ 4,\ 4,\ 4,\ 4\}  \\
6 & Australia Credit & AC& 6 & 8& 1& \ 7 & \ 690 & \ 2 & \{2,\ 2,\ 2,\ 2,\ 3,\ 14,\ 8,\ 3\} \\	
7 & Soybean & SB& 0 & 35&  35& \ 0 & \ 266 & 15 & \{7,\ 2,\ 3,\ 3,\ 2,\ 4,\ 4,\ 3,\ 3,\ 3,\ 2,\ 2,\  3,\ 3,\ 3,\ 2,\ 2,\   \\	
  &  & & &  &&  & && 3,\ 2,\ 2,\ 4,\ 4,\ 2,\ 2,\  2,\ 3,\ 2,\ 3,\ 4,\ 2,\ 2,\ 2,\ 2,\ 2,\ 3\}  \\
8 & Solar Flare & SF& 0 & 9& 9& \ 0 &\ 323 & \ 6 & \{6,\ 4,\ 2,\ 3,\ 2,\ 2,\ 2,\ 2,\ 2\} \\	
9 & Tic-Tac-Toe  & T3& 0 & 9& 9& \ 0  & \ 958 & \ 2 & \{3,\ 3,\ 3,\ 3,\ 3,\ 3,\ 3,\ 3,\ 3\} \\	
10& Hayes-Roth & HR& 0 & 4& 2& \ 2 & \ 132 & \ 3 & \{3,\ 4,\ 4,\ 4\} \\	
11& Lymphography  & LG& 0 & 18&  15& \ 3 & \ 148 & \ 4 & \{3,\ 2,\ 2,\ 2,\ 2,\ 2,\ 2,\ 2,\ 3,\ 4,\ 4,\ 8,\ 3,\ 2,\ 2,\ 2,\ 4,\ 8\}  \\
12& Mushroom  & MR& 0 & 20&  20& \ 0 & 8124 & \ 2 & \{6,\ 4,\ 10,\ 2,\ 9,\ 2,\ 2,\ 2,\ 12,\ 2,\ 4,\ 4,\ 9,\ 9,\ 4,\ 3,\ 5,\ 9,\ 6,\ 7\} \\
{13} & {Lecturer Evaluation} & {LE} & {0} & {4} & {0} & {4} & {1000} & {5} & {\{5, 5, 5, 5\}} \\
{14} & {Social Works} & {SW} & {0} & {10} & {0} & {10} & {1000} & {4} & {\{3, 3, 4, 3, 4, 2, 3, 3, 3, 3\}} \\
\bottomrule
\end{tabular}}
\end{table}

\subsubsection{Experimental Design}

Three types of experiments have been designed to evaluate the proposed HARR-V, HARR-M, and their key technical components. We first compare the clustering performance with the conventional and state-of-the-art counterparts to illustrate the superiority of the proposed approaches. Then, two types of ablation studies are conducted to evaluate the effectiveness of key technical contributions of this work. The total object-cluster distances of HARR-V and HARR-M are also plotted to show their convergence and effectiveness.

\subsubsection{{Datasets}}

14 datasets are utilized for the experimental evaluation. All the datasets are real public ones obtained from the UCI machine learning repository\footnote{https://archive.ics.uci.edu} \cite{uci} to ensure the reliability and reproducibility of the experiments. Statistics of the datasets are shown in Table~\ref{tb:sta_cate}. The former six mixed datasets are utilized to validate the superiority of the proposed approaches in mixed data clustering. The latter eight categorical datasets are used to specifically verify the effectiveness of the proposed attribute representation and learning mechanisms, since the handling of categorical attributes is the key to mixed data clustering.

Categorical attributes of the 14 datasets can be divided into nominal type and ordinal type, and are with different numbers of possible values. All this statistical information is provided in Table~\ref{tb:sta_cate}. {Datasets AP, DT, AC, HR, LG, LE, and SW contain potential ordinal attributes, which are indicated by the official description files of the datasets, or the values of a categorical attribute have an obvious order relationship, e.g., an attribute with possible values \{low, mid, high\}.}

\begin{table*}[h]
\caption{{ARI performance of 12 different clustering approaches on six mixed datasets. The interval of ARI is $[-1,1]$, and a larger value of ARI indicates a better clustering performance. The best and second-best results on each dataset are highlighted in \textbf{bold} and \underline{underlined}, respectively. The experimental results of KMD and KPT are combined and reported as ``KMD/KPT''.}}
\label{tb:clustering_mix_ar}
\centering
\resizebox{0.8\columnwidth}{!}{
\begin{tabular}{l|cccccc}
\toprule	
Approach & 	DS& 	HF& 	AA& 	AP& 	DT& 	AC\\	
\midrule	
KMD/KPT& 0.2881$\pm$0.2718	& 0.0693$\pm$0.0662	& -0.0027$\pm$0.0125	& -0.0068$\pm$0.0129	& 0.4839$\pm$0.1044	& 0.3253$\pm$0.1025\\	
OHE+OC& 0.2745$\pm$0.3484	& 0.0014$\pm$0.0082	& -0.0042$\pm$0.0087	& -0.0162$\pm$0.0080	& 0.6402$\pm$0.1597	& 0.1412$\pm$0.2065\\	
SBC& 0.1728$\pm$0.1719	& -0.0022$\pm$0.0067	& -0.0106$\pm$0.0038	& -0.0073$\pm$0.0005	& 0.5837$\pm$0.1508	& 0.3622$\pm$0.0856\\	
JDM& 0.2243$\pm$0.1923	& 0.0039$\pm$0.0099	& -0.0028$\pm$0.0137	& -0.0043$\pm$0.0162	& 0.5926$\pm$0.1446	& 0.1451$\pm$0.1376\\	
CMS& 0.3815$\pm$0.1064	& 0.0782$\pm$0.0066	& 0.0018$\pm$0.0142	& -0.0034$\pm$0.0125	& 0.5099$\pm$0.0970	& 0.1530$\pm$0.0165\\	
UDM& \underline{0.3943$\pm$0.2863}	& \underline{0.0788$\pm$0.0061}	& 0.0069$\pm$0.0208	& -0.0005$\pm$0.0293	& 0.6207$\pm$0.1285	& \underline{0.3769$\pm$0.0436}\\	
HOD& \underline{0.3943$\pm$0.2863}	& \underline{0.0788$\pm$0.0061}	& 0.0018$\pm$0.0142	& -0.0016$\pm$0.0258	& 0.6207$\pm$0.1285	& \underline{0.3769$\pm$0.0436}\\	
{GWD} & {0.3178$\pm$0.2130} & {0.0951$\pm$0.1146} & {0.0422$\pm$0.2001} & {-0.0089$\pm$0.0303} & {0.0232$\pm$0.0185} & {0.2673$\pm$0.0592} \\
{GBD} & {0.5121$\pm$0.0962} & {0.0770$\pm$0.0051} & {0.0025$\pm$0.0216} & {-0.0044$\pm$0.0229} & {0.6191$\pm$0.1216} & {0.3683$\pm$0.0401} \\
{FBD} & {0.1799$\pm$0.2544} & {0.0717$\pm$0.0745} & {-0.0110$\pm$0.0051} & {\underline{0.0030$\pm$0.0000}} & {\underline{0.6439$\pm$0.1311}} & {0.0597$\pm$0.0412} \\
HARR-V& 0.3849$\pm$0.1897	& \textbf{0.0882$\pm$0.0000}	& \underline{0.0278$\pm$0.0283}	& -0.0016$\pm$0.0456	& 0.6403$\pm$0.1694	& 0.3453$\pm$0.0000\\	
HARR-M& \textbf{0.4599$\pm$0.2409}	& \textbf{0.0882$\pm$0.0000}	& \textbf{0.0289$\pm$0.0285}	& \textbf{0.0412$\pm$0.0534}	& \textbf{0.6826$\pm$0.1429}	& \textbf{0.4200$\pm$0.0000}\\	
\bottomrule	
\end{tabular} }
\end{table*}

{The utilized datasets come from many fields including finance, medicine, biology, etc.;} The proportion of different types of attributes varies greatly in the datasets; The scale of the datasets is not large, thus often cannot provide sufficient statistical information. {Therefore, analyzing these datasets remains very challenging for state-of-the-art techniques.}

\begin{table*}[h]
\caption{{ARI performance of 12 different clustering approaches on eight categorical datasets. The interval of ARI is $[-1,1]$, and a larger value of ARI indicates a better clustering performance. The best and second-best results on each dataset are highlighted in \textbf{bold} and \underline{underlined}, respectively. The experimental results of KMD and KPT are combined and reported as ``KMD/KPT''.}}
\label{tb:clustering_cate_ar}
\centering
\resizebox{1\columnwidth}{!}{
\begin{tabular}{l|cccccccc}
\toprule	
Approach & 	SB& 	SF& 	T3& 	HR& 	LG& 	MR & LE & SW\\	
\midrule	
KMD/KPT& 0.3185$\pm$0.0362	& 0.2083$\pm$0.0649	& 0.0247$\pm$0.0318	& -0.0054$\pm$0.0085	& 0.0683$\pm$0.0986	& 0.3420$\pm$0.2311 &{0.0353$\pm$0.0180} &{0.0485$\pm$0.0168}\\	
OHE+OC& 0.3913$\pm$0.0446	& 0.1961$\pm$0.0509	& 0.0109$\pm$0.0098	& 0.0099$\pm$0.0252	& 0.0892$\pm$0.0563	& 0.4629$\pm$0.2502&{0.0343$\pm$0.0207} &{0.0004$\pm$0.0017}\\	
SBC& 0.3871$\pm$0.0254	& 0.1483$\pm$0.0333	& \underline{0.0297$\pm$0.0182}	& -0.0097$\pm$0.0098	& 0.0065$\pm$0.0029	& 0.4591$\pm$0.2494&{0.0295$\pm$0.0127} &{0.0577$\pm$0.0115}\\	
JDM& 0.3826$\pm$0.0366	& 0.1225$\pm$0.0355	& 0.0244$\pm$0.0267	& -0.0011$\pm$0.0086	& 0.0395$\pm$0.0602	& 0.3191$\pm$0.1839&{0.0359$\pm$0.0115} &{0.0506$\pm$0.0191}\\	
CMS& 0.3356$\pm$0.0449	& 0.2972$\pm$0.0660	& 0.0171$\pm$0.0217	& -0.0011$\pm$0.0063	& 0.0824$\pm$0.0759	& 0.4210$\pm$0.2690&{0.0384$\pm$0.0193} &{0.0608$\pm$0.0159}\\	
UDM& 0.3673$\pm$0.0411	& 0.2800$\pm$0.0703	& 0.0178$\pm$0.0250	& 0.0340$\pm$0.0354	& 0.0860$\pm$0.0788	& 0.3675$\pm$0.2519&{0.0425$\pm$0.0185} &{0.0620$\pm$0.0132}\\	
HOD& 0.3686$\pm$0.0397	& 0.2799$\pm$0.0848	& 0.0178$\pm$0.0250	& 0.0325$\pm$0.0386	& 0.0762$\pm$0.0854	& 0.3626$\pm$0.2348&{0.0425$\pm$0.0185} &{0.0661$\pm$0.0130}\\	
{GWD} & {0.2530$\pm$0.0520} & {0.1829$\pm$0.0551} & {0.0177$\pm$0.0210} & {0.0103$\pm$0.0222} & {0.0619$\pm$0.0590} & {0.1371$\pm$0.1649} & {0.0361$\pm$0.0207} & {0.0379$\pm$0.0137} \\
{GBD} & {0.3574$\pm$0.0416} & {0.2395$\pm$0.0688} & {0.0147$\pm$0.0223} & {0.0328$\pm$0.0413} & {0.1012$\pm$0.0471} & {0.3197$\pm$0.2474} & {\underline{0.0435$\pm$0.0190}} & {0.0627$\pm$0.0128} \\
{FBD} & {0.3377$\pm$0.0457} & {0.1996$\pm$0.0552} & {0.0198$\pm$0.0013} & {0.0390$\pm$0.0217} & {0.1530$\pm$0.0274} & {0.3795$\pm$0.2169} & {0.0323$\pm$0.0176} & {0.0598$\pm$0.0048} \\	
HARR-V& \underline{0.4196$\pm$0.0441}	& \underline{0.3216$\pm$0.0628}	& 0.0211$\pm$0.0273	& \textbf{0.0580$\pm$0.0211}	& \underline{0.1687$\pm$0.0226}	& \underline{0.5667$\pm$0.1742}&{\textbf{0.0615$\pm$0.0415}} &{\textbf{0.0880$\pm$0.0136}}\\	
HARR-M& \textbf{0.4367$\pm$0.0495}	& \textbf{0.3254$\pm$0.0748}	& \textbf{0.0338$\pm$0.0252}	& \underline{0.0524$\pm$0.0228}	& \textbf{0.1849$\pm$0.0000}	& \textbf{0.6122$\pm$0.0678}&{\textbf{0.0615$\pm$0.0345}} &{\underline{0.0778$\pm$0.0215}}\\	
\bottomrule	
\end{tabular}}
\end{table*}
\subsubsection{Counterparts}

Twelve clustering approaches are compared including ten counterparts and the proposed HARR-V and HARR-M. The conventional k-means \cite{kms} clustering algorithm is combined with One-Hot Encoding (OHE, for nominal attributes) and Order Coding (OC, for ordinal attributes) to be a counterpart (OHE+OC), where OC encodes the values of an ordinal attribute into their order values, and then normalizes them. OHE+OC is one of the most straightforward and common solutions for practical mixed data clustering. Two other variants of k-means, i.e., k-modes (KMD) \cite{kmd} and k-prototypes (KPT) \cite{kpt} are also chosen. {As KMD is feasible for pure categorical data and KPT is designed for mixed data, the experimental results of KMD and KPT are combined and reported as ``KMD/KPT''. GoWer's Distance(GWD) \cite{Gower}, Structure-Based Categorical (SBC) data encoding \cite{sbc}, Jia's Distance Metric (JDM) \cite{jdm}, Coupled Metric Similarity (CMS) \cite{cms}, Unified Distance Metric (UDM) \cite{udm}, HOmogeneous Distance (HOD) metric \cite{hd-ndw}, Graph-Based Distance (GBD) metric \cite{Adc}, and Forest-Based Distance(FBD) metric \cite{COForest}, combined with KMD and KPT for categorical and mixed data, respectively, have also been chosen as counterparts.}

\subsubsection{Validity Indices}

To comprehensively evaluate the clustering performance of compared approaches, we adopt two popular validity indices, i.e., Adjusted Rand Index (ARI) \cite{ex2} and Clustering Accuracy (CA) \cite{ex1} for the evaluation. CA is a conventional index in the interval $[0,1]$, which first searches the optimal mapping between label and clustering results, and then measures the accuracy. ARI in the interval $[-1,1]$ is a more discriminative index based on the conventional Rand Index \cite{ex3}\cite{ex4}. {For both ARI and CA, a larger value indicates a better clustering performance.}

\subsection{Comparison of Clustering Performance}\label{subsct:ex_clustering}

\begin{table*}[!t]
\caption{{CA performance of 12 different clustering approaches on six mixed datasets. The interval of CA is $[0,1]$, and a larger value of CA indicates a better clustering performance. The best and second-best results on each dataset are highlighted in \textbf{bold} and \underline{underlined}, respectively. The experimental results of KMD and KPT are combined and reported as ``KMD/KPT''.}}
\label{tb:clustering_mix_ca}
\centering
\resizebox{0.8\columnwidth}{!}{
\begin{tabular}{l|cccccc}
\toprule	
Approach & 	DS& 	HF& 	AA& 	AP& 	DT& 	AC \\	
\midrule	
KMD/KPT& 0.7375$\pm$0.1352	& 0.6224$\pm$0.0668	& 0.5361$\pm$0.0330	& 0.5299$\pm$0.0198	& 0.5964$\pm$0.0893	& 0.7784$\pm$0.0658\\	
OHE+OC& 0.7196$\pm$0.1527	& 0.5415$\pm$0.0330	& 0.5322$\pm$0.0283	& 0.5079$\pm$0.0097	& 0.6753$\pm$0.1240	& 0.6267$\pm$0.1442\\	
SBC& 0.6846$\pm$0.1085	& 0.5251$\pm$0.0235	& 0.5173$\pm$0.0067	& 0.5410$\pm$0.0024	& 0.6264$\pm$0.1276	& 0.7946$\pm$0.0646\\	
JDM& 0.7154$\pm$0.1105	& 0.5580$\pm$0.0259	& 0.5385$\pm$0.0391	& 0.5479$\pm$0.0326	& 0.6536$\pm$0.1137	& 0.6633$\pm$0.1031\\	
CMS& 0.7686$\pm$0.0479	& 0.6436$\pm$0.0054	& 0.5673$\pm$0.0099	& 0.5397$\pm$0.0284	& 0.6096$\pm$0.0768	& 0.6971$\pm$0.0102\\	
UDM& 0.7900$\pm$0.1302	& \underline{0.6441$\pm$0.0050}	& \underline{0.5745$\pm$0.0202}	& 0.5614$\pm$0.0402	& 0.6679$\pm$0.0995	& \underline{0.8069$\pm$0.0175}\\	
HOD& 0.7900$\pm$0.1302	& \underline{0.6441$\pm$0.0050}	& 0.5673$\pm$0.0099	& 0.5569$\pm$0.0349	& 0.6679$\pm$0.0995	& \underline{0.8069$\pm$0.0175}\\
{GWD} & {0.7433$\pm$0.1137} & {0.6492$\pm$0.0950} & {0.5654$\pm$0.1039} & {0.5550$\pm$0.0441} & {0.2824$\pm$0.0261} & {0.7570$\pm$0.0336} \\
{GBD} & {0.7767$\pm$0.0447} & {0.6428$\pm$0.0043} & {0.5611$\pm$0.0285} & {\underline{0.5640$\pm$0.0321}} & {0.6662$\pm$0.0899} & {0.8034$\pm$0.0161} \\
{FBD} & {0.6908$\pm$0.1371} & {0.6284$\pm$0.0602} & {0.5120$\pm$0.0162} & {0.5450$\pm$0.0000} & {\underline{0.6828$\pm$0.1001}} & {0.6154$\pm$0.0494} \\
	
HARR-V& \underline{0.7942$\pm$0.1084}	& \textbf{0.6522$\pm$0.0000}	& \textbf{0.6062$\pm$0.0305}	& 0.5550$\pm$0.0561	& 0.6793$\pm$0.1411	& 0.7942$\pm$0.0000\\	
HARR-M& \textbf{0.8187$\pm$0.1247}	& \textbf{0.6522$\pm$0.0000}	& \textbf{0.6062$\pm$0.0302}	& \textbf{0.6111$\pm$0.0535}	& \textbf{0.7161$\pm$0.1130}	& \textbf{0.8246$\pm$0.0000}\\	
\bottomrule	
\end{tabular}}
\end{table*}

\begin{table*}[!t]
\caption{{CA performance of 12 different clustering approaches on eight categorical datasets. The interval of CA is $[0,1]$, and a larger value of CA indicates a better clustering performance. The best and second-best results on each dataset are highlighted in \textbf{bold} and \underline{underlined}, respectively. The experimental results of KMD and KPT are combined and reported as ``KMD/KPT''. }}
\label{tb:clustering_cate_ca}
\centering
\resizebox{1\columnwidth}{!}{
\begin{tabular}{l|cccccccc}
\toprule	
Approach & 	SB& 	SF& 	T3& 	HR & LG & MR & LE & SW	\\	
\midrule	
KMD/KPT& 0.5021$\pm$0.0436	& 0.4707$\pm$0.0523	& 0.5692$\pm$0.0499	& 0.3826$\pm$0.0240	& 0.6028$\pm$0.0933	& 0.7645$\pm$0.1282 &{0.3276$\pm$0.0344} & {0.3777$\pm$0.0348}\\	
OHE+OC& 0.5348$\pm$0.0433	& 0.4506$\pm$0.0393	& 0.5548$\pm$0.0266	& 0.3989$\pm$0.0561	& 0.6444$\pm$0.0568	& 0.8167$\pm$0.1277 &{0.3455$\pm$0.1500} &{0.0210$\pm$0.0941}\\
SBC& 0.5368$\pm$0.0300	& 0.4130$\pm$0.0284	& \underline{0.5825$\pm$0.0378}	& 0.3591$\pm$0.0349	& 0.5588$\pm$0.0062	& 0.8145$\pm$0.1295 &{0.3172$\pm$0.0263} &{0.3812$\pm$0.0195}\\	
JDM& 0.5417$\pm$0.0412	& 0.4046$\pm$0.0306	& 0.5664$\pm$0.0455	& 0.3932$\pm$0.0269	& 0.5842$\pm$0.0707	& 0.7613$\pm$0.1101 &{0.3142$\pm$0.0240} &{0.3738$\pm$0.0274}\\	
CMS& 0.5177$\pm$0.0476	& 0.5042$\pm$0.0484	& 0.5551$\pm$0.0399	& 0.3962$\pm$0.0209	& 0.6282$\pm$0.0783	& 0.7958$\pm$0.1368 &{0.3274$\pm$0.0379} &{0.3866$\pm$0.0292}\\	
UDM& 0.5376$\pm$0.0391	& 0.5020$\pm$0.0532	& 0.5544$\pm$0.0433	& 0.4470$\pm$0.0521	& 0.6289$\pm$0.0828	& 0.7757$\pm$0.1296 &{0.3324$\pm$0.0391} &{0.3819$\pm$0.0231}\\	
HOD& 0.5393$\pm$0.0405	& 0.5050$\pm$0.0641	& 0.5543$\pm$0.0434	& 0.4405$\pm$0.0568	& 0.6190$\pm$0.0833	& 0.7780$\pm$0.1190 &{0.3324$\pm$0.0391} & {0.3880$\pm$0.0257}\\
{GWD} & {0.4462$\pm$0.0505} & {0.4515$\pm$0.0550} & {0.5643$\pm$0.0389} & {0.4163$\pm$0.0384} & {0.6176$\pm$0.0611} & {0.6430$\pm$0.1208} & {0.3352$\pm$0.0319} & {0.3721$\pm$0.0254} \\
{GBD} & {0.5295$\pm$0.0401} & {0.4873$\pm$0.0577} & {0.5478$\pm$0.0419} & {0.4383$\pm$0.0599} & {0.6588$\pm$0.0447} & {0.7523$\pm$0.1310} & {0.3341$\pm$0.0320} & {0.3810$\pm$0.0199} \\
{FBD} & {0.5199$\pm$0.0459} & {0.4565$\pm$0.0444} & {0.5784$\pm$0.0007} & {0.3905$\pm$0.0376} & {0.6982$\pm$0.0211} & {0.7827$\pm$0.0604} & {0.2912$\pm$0.0233} & {0.3887$\pm$0.0172} \\

HARR-V& \underline{0.5492$\pm$0.0408}	& \underline{0.5127$\pm$0.0540}	& 0.5688$\pm$0.0386	& \textbf{0.4883$\pm$0.0409}	& \underline{0.7085$\pm$0.0138}	& \underline{0.8642$\pm$0.0976} &{\underline{0.3658$\pm$0.0585}} &{\underline{0.4089$\pm$0.0246}}\\
HARR-M& \textbf{0.5564$\pm$0.0533}	& \textbf{0.5139$\pm$0.0605}	& \textbf{0.5863$\pm$0.0405}	& \underline{0.4674$\pm$0.0416}	& \textbf{0.7183$\pm$0.0000}	& \textbf{0.8905$\pm$0.0246} &{\textbf{0.3687$\pm$0.0463}} & {\textbf{0.4095$\pm$0.0260}}\\
\bottomrule	
\end{tabular}}
\end{table*}

Clustering performance of the compared approaches is shown in {Tables}~\ref{tb:clustering_mix_ar} - \ref{tb:clustering_cate_ca}. In general, HARR-M outperforms all the counterparts, and HARR-V is very competitive in most comparisons. More detailed observations are provided in the following: (1) It can be seen that HARR-M outperforms HARR-V in general. This is because that HARR-M more finely considers the contribution of representation to the formation of different clusters; (2) Superiority of the performance of HARR-M and HARR-V is more obvious on mixed data than on categorical data because the advantage of the proposed representation method that homogeneously represents heterogeneous numerical and categorical attributes will surely be weakened on datasets composed of only categorical attributes; (3) Results in Tables~\ref{tb:clustering_cate_ar} and~\ref{tb:clustering_cate_ca} show that both HARR-V and HARR-M perform well on categorical datasets.

Based on the results shown in Tables~\ref{tb:clustering_mix_ar} - \ref{tb:clustering_cate_ca}, we also compute: (1) the average ARI and CA performance, and (2) the ARI and CA performance rankings, on the datasets for all the counterparts to facilitate an intuitive comparison in Figure~\ref{fig:sta_clustering}. To better organize the histogram results, KMD/KPT, OHE+OC, SBC, JDM, CMS, UDM, HOD, HARR-V, and HARR-M, are interchangeably denoted as A, B, ..., G, H-V, H-M, respectively, hereinafter. The standard deviation of average ARI and CA are not reported because the clustering performance on different datasets may be very different. It can be observed that HARR-V and HARR-M perform obviously better than the other counterparts. Moreover, HARR-M always ranks the first with a very small standard deviation, which indicates {its} superiority and stability.

\begin{figure}[h]
\centering
\subfigure[\footnotesize{Average ARI performance and average ranks.}]{
    \includegraphics[width=1.42in]{ 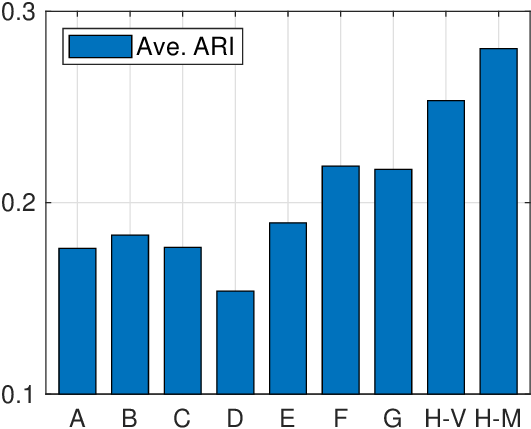}
    \centering
    \includegraphics[width=1.4in]{ 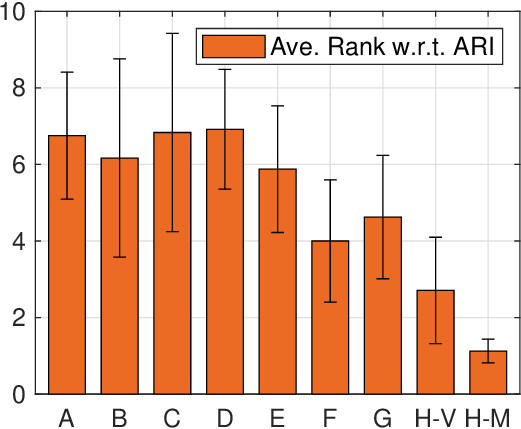}
}
\subfigure[\footnotesize{Average CA performance and average ranks.}]{
    \includegraphics[width=1.42in]{ 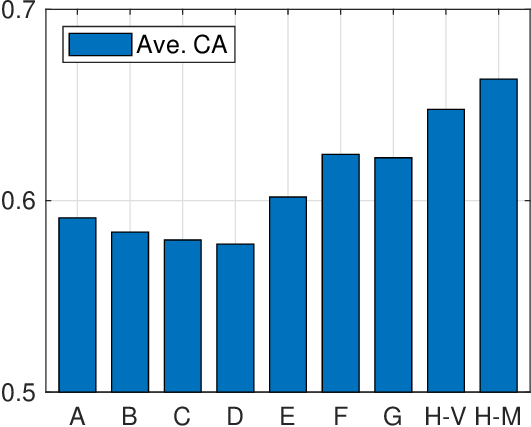}
    \centering
    \includegraphics[width=1.4in]{ 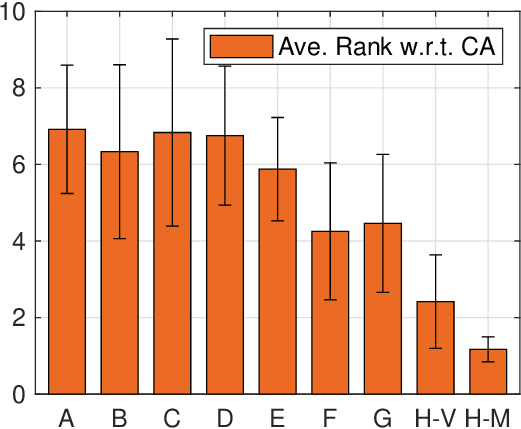}
}
 
\caption{{Average performance and ranks of different approaches based on the results in Tables~\ref{tb:clustering_mix_ar} - \ref{tb:clustering_cate_ca}. The methods KMD/KPT, OHE+OC, SBC, JDM, CMS, UDM, HOD, HARR-V, and HARR-M are denoted as A, B, C, D, E, F, G, H-V, and H-M, respectively.}}	
\label{fig:sta_clustering}	
\end{figure}
\begin{figure}[h]	
 
\centerline{\includegraphics[width=5.8in]{ 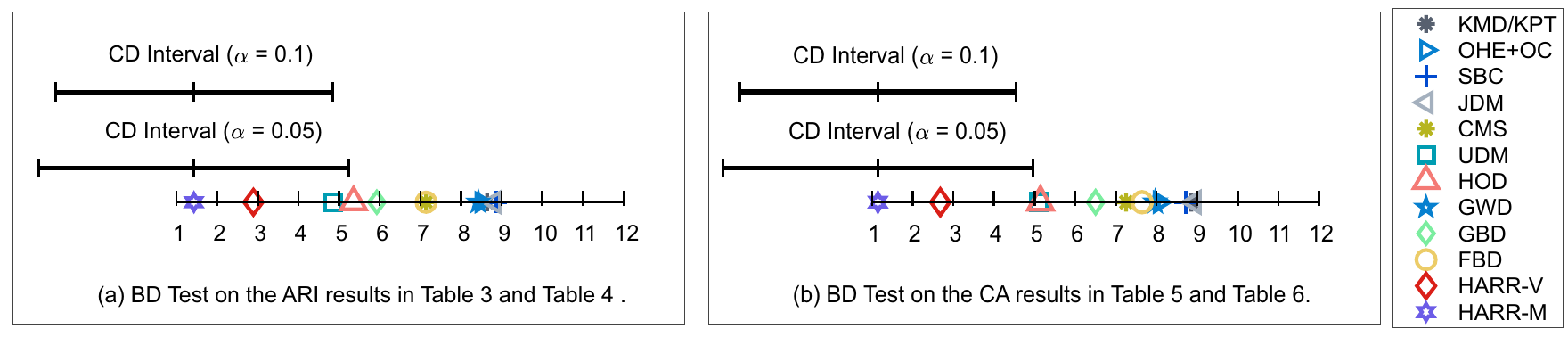}}  
\caption{{Results of the two-tailed BD tests on ARI (a) and CA (b) performances across 12 methods on 14 datasets, based on the results from  {Tables}~\ref{tb:clustering_mix_ar}-\ref{tb:clustering_cate_ca}. Confidence intervals are set at 0.95 (\(\alpha = 0.05\)) and 0.9 (\(\alpha = 0.1\)), the corresponding Critical Difference (CD) intervals CD\_95\%=3.8158 and CD\_90\%=3.4069, respectively.  }}	
\label{fig:sig_all}
\end{figure}

{To confirm that improvements over baseline methods are significant and consistent, we first implement the Friedman~\cite{friedman} test on the ARI and CA results in Tables~\ref{tb:clustering_mix_ar} - \ref{tb:clustering_cate_ca}.} The corresponding p-values are 0.00000011 and 0.0000013 respectively, both passing the test under the confidence of 99\% (p-value $= 0.01$). On this basis, results of the two-tailed BD test~\cite{r1test} with confidence intervals 0.95 ($\alpha = 0.05$) and 0.9 ($\alpha = 0.1$) are shown in Figure \ref{fig:sig_all}, with the corresponding Critical Difference (CD) intervals CD\_95\%=3.8158 and CD\_90\%=3.4069, respectively, for comparing 12 methods across 14 datasets. As can be seen from Figure~\ref{fig:sig_all}, all 10 compared methods fall outside the right boundary of the CD intervals, except for the UDM method w.r.t. ARI performance under $\alpha=0.05$. But it is worth mentioning that UDM is very close to the boundary of $\alpha=0.05$ and stays outside the boundary of $\alpha=0.1$. In general, the test results indicate that our method, HARR-M,  significantly outperforms all {the} other counterparts except for the other version of our method, HARR-V, which achieves comparable performance.

%


\begin{table}[h]
\caption{{ARI performance of five ablated HARR versions on 14 datasets. The interval of ARI is $[-1,1]$, and a larger value of ARI indicates a better clustering performance. $\overline{AR}$ row reports the average ranks.}}
\label{tb:ablation_component_ar}
\centering
\resizebox{0.7\columnwidth}{!}{
\begin{tabular}{c|ccccc}
\toprule	
Data & 	KMD/KPT& 	BD& 	HAR& 	HARR-V& 	HARR-M\\	
\midrule	
DS& 0.2881$\pm$0.2718	& 0.3374$\pm$0.2798	& 0.3374$\pm$0.2798	& 0.3849$\pm$0.1897	& 0.4599$\pm$0.2409\\	
HF& 0.0693$\pm$0.0662	& 0.0793$\pm$0.0181	& 0.0793$\pm$0.0181	& 0.0882$\pm$0.0000	& 0.0882$\pm$0.0000\\	
AA& -0.0027$\pm$0.0125	& -0.0030$\pm$0.0106	& 0.0430$\pm$0.0313	& 0.0278$\pm$0.0283	& 0.0289$\pm$0.0285\\	
AP& -0.0068$\pm$0.0129	& -0.0020$\pm$0.0288	& -0.0002$\pm$0.0300	& -0.0016$\pm$0.0456	& 0.0412$\pm$0.0534\\	
DT& 0.4839$\pm$0.1044	& 0.6322$\pm$0.1343	& 0.6322$\pm$0.1343	& 0.6403$\pm$0.1694	& 0.6826$\pm$0.1429\\	
AC& 0.3253$\pm$0.1025	& 0.3773$\pm$0.0455	& 0.3773$\pm$0.0455	& 0.3453$\pm$0.0000	& 0.4200$\pm$0.0000\\	
SB& 0.3185$\pm$0.0362	& 0.3686$\pm$0.0397	& 0.3912$\pm$0.0355	& 0.4196$\pm$0.0441	& 0.4367$\pm$0.0495\\	
SF& 0.2083$\pm$0.0649	& 0.2799$\pm$0.0848	& 0.3090$\pm$0.0879	& 0.3216$\pm$0.0628	& 0.3254$\pm$0.0748\\	
T3& 0.0247$\pm$0.0318	& 0.0178$\pm$0.0250	& 0.0196$\pm$0.0226	& 0.0211$\pm$0.0273	& 0.0338$\pm$0.0252\\	
HR& -0.0054$\pm$0.0085	& 0.0325$\pm$0.0386	& 0.0244$\pm$0.0177	& 0.0580$\pm$0.0211	& 0.0524$\pm$0.0228\\	
LG& 0.0683$\pm$0.0986	& 0.0762$\pm$0.0854	& 0.1187$\pm$0.0676	& 0.1687$\pm$0.0226	& 0.1849$\pm$0.0000\\	
MR& 0.3420$\pm$0.2311	& 0.3626$\pm$0.2348	& 0.4165$\pm$0.2588	& 0.5667$\pm$0.1742	& 0.6122$\pm$0.0678\\	
{LE} & {0.0353$\pm$0.0180} & {0.0425$\pm$0.0185} & {0.0439$\pm$0.0167} & {0.0615$\pm$0.0415} & {0.0615$\pm$0.0345} \\
{SW} & {0.0485$\pm$0.0168} & {0.0661$\pm$0.0130} & {0.0809$\pm$0.0135} & {0.0880$\pm$0.0136} & {0.0778$\pm$0.0215} \\

\midrule\midrule	
$\overline{AR}$ & {4.7143$\pm$0.8254} & {3.8571$\pm$0.6630} & {2.9286$\pm$0.8287} & {2.1071$\pm$0.8810} & {1.3929$\pm$0.6257} \\	
\bottomrule	

\end{tabular}
}
\end{table}

\subsection{Ablation Studies}\label{subsct:ex_ablation}

To evaluate the effectiveness of the key technical components proposed in this paper, the performance of different ablated versions of HARR-V and HARR-M is compared in Tables~\ref{tb:ablation_component_ar} and~\ref{tb:ablation_component_ca}. We combine the Homogeneous Attribute Representation (HAR) proposed in Section~\ref{subsct:pbr} with KMD/KPT (without representation learning), and compare its performance with HARR-V and HARR-M to illustrate the effectiveness of the weights updating mechanisms proposed in Section~\ref{subsct:ls}. Performance of the Base Distance (BD) described by Eq.~(\ref{eq:cpddiff}) and Figure~\ref{fig:dist_indication} combined with KMD/KPT is compared with HAR to verify the effectiveness of the projection mechanism of the representation. Performance of KMD/KPT is also reported for completeness.

\newcommand{\mylinespace}{41mm}

In Tables~\ref{tb:ablation_component_ar} and~\ref{tb:ablation_component_ca}, to make the results easy to observe, the row of Average Rank ($\overline{AR}$) reports the average performance ranks of compared approaches, and the row of Improved Percentage ($\overline{IP}$) reports the average improvement each approach achieves over its left approach. Since ARI is in the interval [-1,1], $\overline{IP}$ results are not reported in Table~\ref{tb:ablation_component_ar}. It can be observed that the performance of KMD/KPT, BD, HAR, HARR-V, and HARR-M increases in general. {More specifically, the following four observations: (1) BD outperforms KMD/KPT, (2) HAR outperforms BD, (3) HARR-V outperforms HAR, and (4) HARR-M outperforms HARR-V, can clearly verify the effectiveness of: (1) base distance adopted by the projection, (2) projection-based representation, (3) weights learning mechanism of HARR-V, and (4) weights learning mechanism of HARR-M, respectively.}

\begin{table}[h]
\caption{{CA performance of five ablated HARR versions on 14 datasets. The interval of CA is $[0,1]$, and a larger value of CA indicates a better clustering performance. $\overline{AR}$ row reports the average ranks. $\overline{IP}$ row reports the improved percentage achieved by the current one w.r.t. the previous one. In Table~\ref{tb:ablation_component_ar}, since ARI is in the interval $[-1,1]$, $\overline{IP}$ results are not reported. }}
\label{tb:ablation_component_ca}
\centering
\resizebox{0.7\columnwidth}{!}{
\begin{tabular}{c|ccccc}
\toprule	
Data & 	KMD/KPT& 	BD& 	HAR& 	HARR-V& 	HARR-M\\	
\midrule	
DS& 0.7375$\pm$0.1352	& 0.7650$\pm$0.1289	& 0.7650$\pm$0.1289	& 0.7942$\pm$0.1084	& 0.8187$\pm$0.1247\\	
HF& 0.6224$\pm$0.0668	& 0.6436$\pm$0.0145	& 0.6436$\pm$0.0145	& 0.6522$\pm$0.0000	& 0.6522$\pm$0.0000\\	
AA& 0.5361$\pm$0.0330	& 0.5404$\pm$0.0302	& 0.6197$\pm$0.0475	& 0.6062$\pm$0.0305	& 0.6062$\pm$0.0302\\	
AP& 0.5299$\pm$0.0198	& 0.5561$\pm$0.0381	& 0.5614$\pm$0.0446	& 0.5550$\pm$0.0561	& 0.6111$\pm$0.0535\\	
DT& 0.5964$\pm$0.0893	& 0.6811$\pm$0.1071	& 0.6811$\pm$0.1071	& 0.6793$\pm$0.1411	& 0.7161$\pm$0.1130\\	
AC& 0.7784$\pm$0.0658	& 0.8070$\pm$0.0182	& 0.8070$\pm$0.0182	& 0.7942$\pm$0.0000	& 0.8246$\pm$0.0000\\	
SB& 0.5021$\pm$0.0436	& 0.5393$\pm$0.0405	& 0.5258$\pm$0.0457	& 0.5492$\pm$0.0408	& 0.5564$\pm$0.0533\\	
SF& 0.4707$\pm$0.0523	& 0.5050$\pm$0.0641	& 0.5053$\pm$0.0770	& 0.5127$\pm$0.0540	& 0.5139$\pm$0.0605\\	
T3& 0.5692$\pm$0.0499	& 0.5543$\pm$0.0434	& 0.5602$\pm$0.0405	& 0.5688$\pm$0.0386	& 0.5863$\pm$0.0405\\	
HR& 0.3826$\pm$0.0240	& 0.4405$\pm$0.0568	& 0.4496$\pm$0.0402	& 0.4883$\pm$0.0409	& 0.4674$\pm$0.0416\\	
LG& 0.6028$\pm$0.0933	& 0.6190$\pm$0.0833	& 0.6687$\pm$0.0551	& 0.7085$\pm$0.0138	& 0.7183$\pm$0.0000\\	
MR& 0.7645$\pm$0.1282	& 0.7780$\pm$0.1190	& 0.7940$\pm$0.1368	& 0.8642$\pm$0.0976	& 0.8905$\pm$0.0246\\	
{LE} & {0.3276$\pm$0.0344} & {0.3324$\pm$0.0391} & {0.3330$\pm$0.0324} & {0.3658$\pm$0.0585} & {0.3687$\pm$0.0463} \\
{SW} & {0.3777$\pm$0.0348} & {0.3880$\pm$0.0257} & {0.4011$\pm$0.0233} & {0.4089$\pm$0.0246} & {0.4095$\pm$0.0260} \\

\midrule\midrule	
$\overline{AR}$ & {4.7857$\pm$0.8018} & {3.6429$\pm$0.6914} & {2.9286$\pm$0.7810} & {2.4286$\pm$0.9579} & {1.2143$\pm$0.4688} \\
$\overline{IP} $& {-} & {5.5546$\%$} & {2.4933$\%$} & {3.5547$\%$} & {2.3478$\%$} \\

\bottomrule	

\end{tabular}
}
\end{table}

{This work distinguishes categorical attributes into nominal and ordinal ones.} To evaluate the necessity of doing so, the clustering performance of approaches is compared by treating potential ordinal attributes in the following two ways: (1) distinguish nominal and ordinal attributes, and (2) treat all the categorical attributes as nominal ones. Note that only the approaches feasible for distinguishing nominal and ordinal attributes are compared here. According to the results visualized in Figure~\ref{fig:switch}, treating nominal and ordinal attributes differently can obviously boost clustering performance in most cases, which confirms the necessity of our distinction between ordinal and nominal attributes.

\subsection{Learning Process Evaluation}

The values of total object-cluster distance $z=\sum_{i=1}^n\sum_{l=1}^{k^*}q_{il}\Phi_w(\textbf{x}_i,\textbf{m}_l)$ produced by HARR-V and HARR-M on the six mixed datasets are plotted in Figure~\ref{fig:obj_val}. The x-axis is the number of learning iterations, and the y-axis is the value of $z$. The circles and boxes indicate that the weights were updated by HARR-V and HARR-M in the corresponding iterations, respectively.

It can be observed from the figures that the $z$ values of HARR-V and HARR-M monotonically decrease on all the datasets, and both the two algorithms always converge quickly within 15 iterations. This indicates that although HARR-V and HARR-M need to learn the representations of $\hat{d}$ attributes ($\hat{d}\geq d$), they are still very efficient, which is a very important characteristic in practical applications. {It can also be observed that the $z$ values of HARR-M are never higher than the $z$ values of HARR-V at convergence, which again indicates the effectiveness of the weights updating strategy of HARR-M.}

\begin{figure*}[h]
\newcommand{\mywd}{1.2in}
\newcommand{\mylwd}{0.188}
\centering
\subfigure{
\begin{minipage}{\mylwd\linewidth}
\centering
\includegraphics[width=\mywd]{ 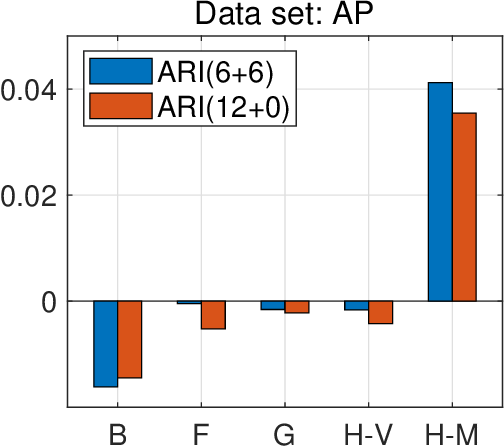}
\end{minipage}}
\subfigure{
\begin{minipage}{\mylwd\linewidth}
\centering
\includegraphics[width=\mywd]{ 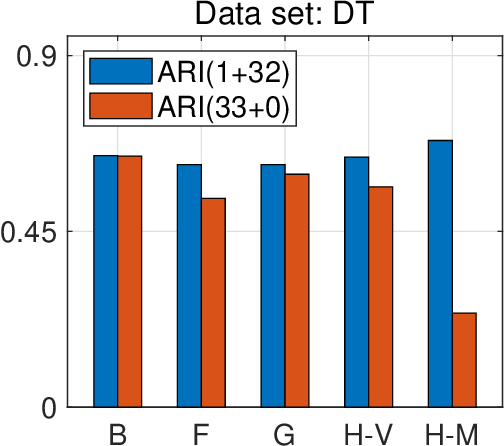}
\end{minipage}}
\subfigure{
\begin{minipage}{\mylwd\linewidth}
\centering
\includegraphics[width=\mywd]{ 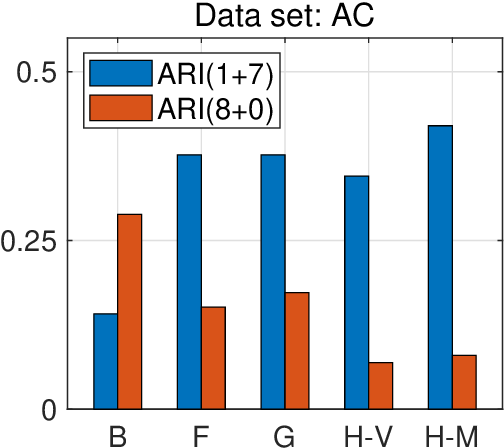}
\end{minipage}}
\subfigure{
\begin{minipage}{\mylwd\linewidth}
\centering
\includegraphics[width=\mywd]{ 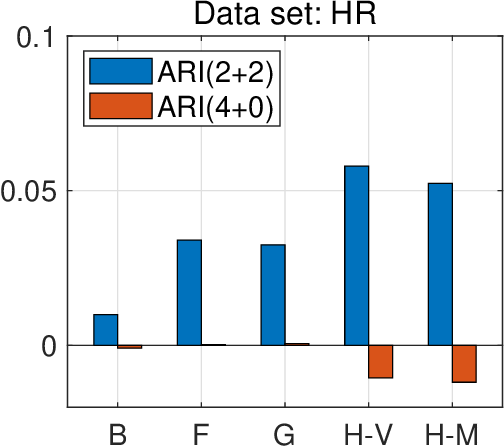}
\end{minipage}}
\subfigure{
\begin{minipage}{\mylwd\linewidth}
\centering
\includegraphics[width=\mywd]{ 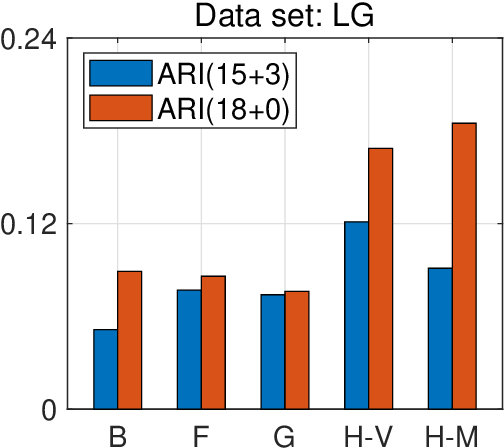}
\end{minipage}}
\subfigure{
\begin{minipage}{\mylwd\linewidth}
\centering
\includegraphics[width=\mywd]{ 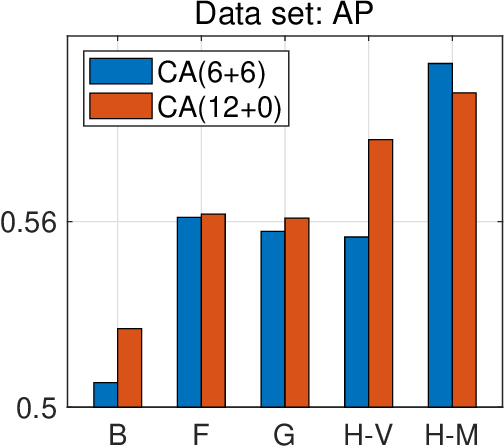}
\end{minipage}}
\subfigure{
\begin{minipage}{\mylwd\linewidth}
\centering
\includegraphics[width=\mywd]{ 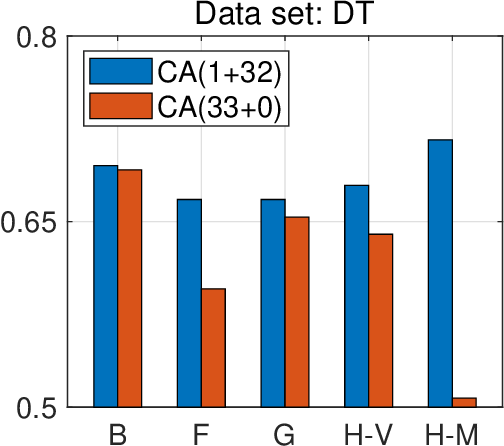}
\end{minipage}}
\subfigure{
\begin{minipage}{\mylwd\linewidth}
\centering
\includegraphics[width=\mywd]{ 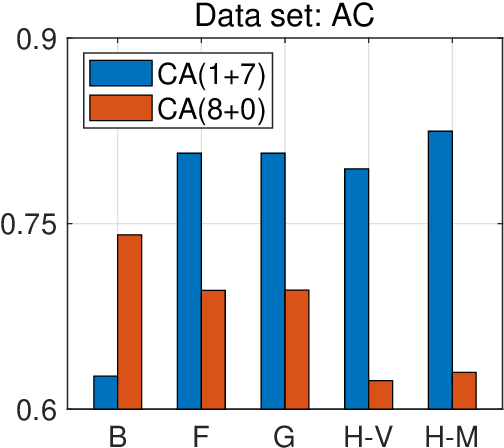}
\end{minipage}}
\subfigure{
\begin{minipage}{\mylwd\linewidth}
\centering
\includegraphics[width=\mywd]{ 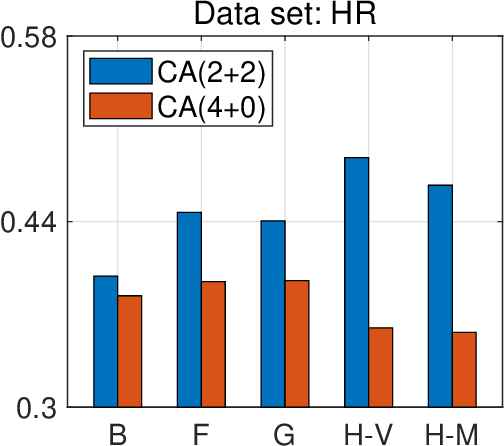}
\end{minipage}}
\subfigure{
\begin{minipage}{\mylwd\linewidth}
\centering
\includegraphics[width=\mywd]{ 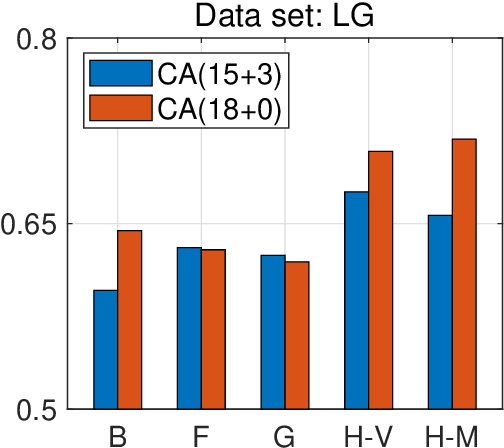}
\end{minipage}}
\caption{{ARI and CA performance of different approaches with and without distinguishing ordinal attributes on AP, DT, AC, HR, and LG datasets. The two numbers in parentheses in the legend report the value of $d_n$+$d_o$ used to treat the dataset.}}	
\label{fig:switch}	
\end{figure*}

\begin{figure*}[!t]
\newcommand{\mywd}{1.76in} 
\newcommand{\mylwd}{0.30} 
\centering
\subfigure{
    \begin{minipage}{\mylwd\linewidth}
        \centering
        \includegraphics[width=\mywd]{ 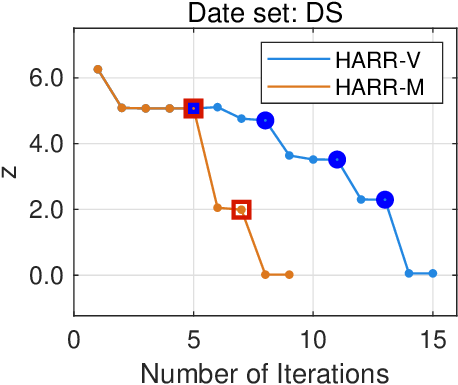}
    \end{minipage}
    \begin{minipage}{\mylwd\linewidth}
        \centering
        \includegraphics[width=\mywd]{ 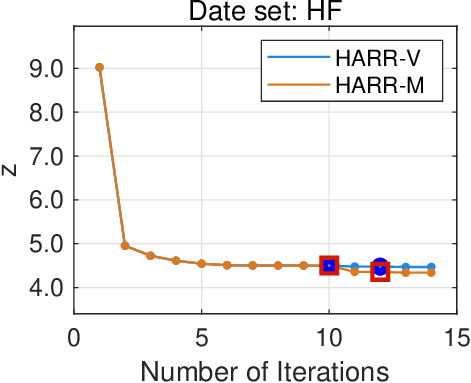}
    \end{minipage}
    \begin{minipage}{\mylwd\linewidth}
        \centering
        \includegraphics[width=\mywd]{ 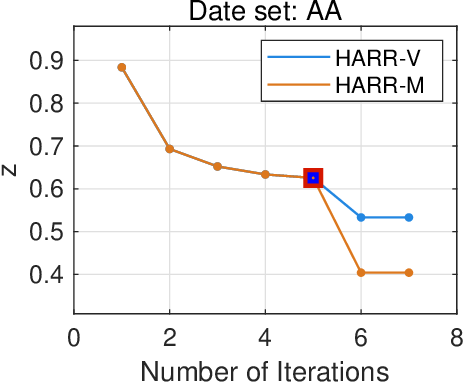}
    \end{minipage}
}
\subfigure{
    \begin{minipage}{\mylwd\linewidth}
        \centering
        \includegraphics[width=\mywd]{ 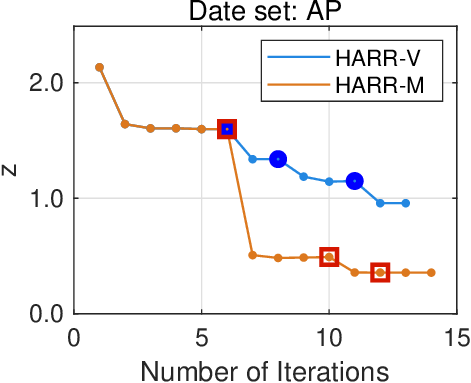}
    \end{minipage}
    \begin{minipage}{\mylwd\linewidth}
        \centering
        \includegraphics[width=\mywd]{ 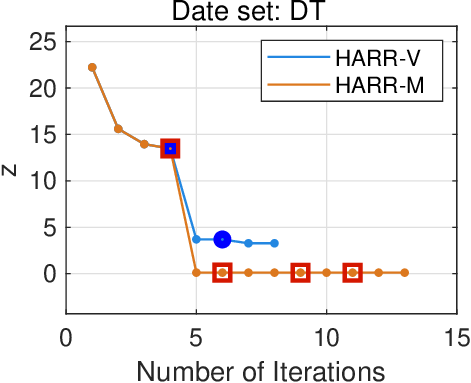}
    \end{minipage}
    \begin{minipage}{\mylwd\linewidth}
        \centering
        \includegraphics[width=\mywd]{ 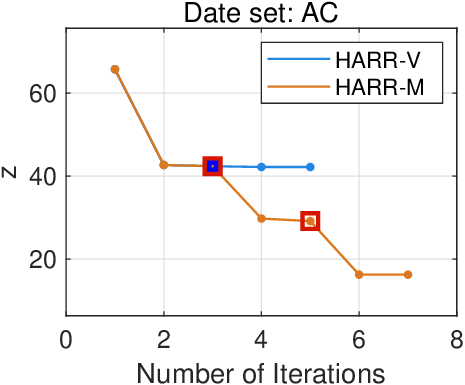}
    \end{minipage}
}
\caption{{Total object-cluster distance $z$ of HARR-V and HARR-M on different mixed data sets. The blue circles and red boxes indicate that the weights were updated by HARR-V and HARR-M, respectively.}}	
\label{fig:obj_val}	
\end{figure*}

\subsection{Efficiency Evaluation}
{
To specifically evaluate the efficiency of HARR-V and HARR-M, we compare their execution time with their state-of-the-art counterparts. To better observe the scalability and verify the time complexity, we compare the approaches on a large randomly generated synthetic dataset with $n = 10,0000$, $d= d^c = 5$, $v^1 = v^2 =\ldots= 5$ and $k=5$. All the compared approaches are implemented by first computing the similarity/distance matrices of each attribute and then performing clustering by reading off the similarities/distances. The time (y-axis) - $n$ (x-axis) curves in Figure~\ref{fig:time_n} are formed by sampling the synthetic dataset using the sampling rates $\varphi =\{0.001, 0.2, 0.4, 0.6, 0.8, 1\}$. The efficiency-related experiments are coded by MATLAB R2021b and implemented by a PC (Intel i7-9700 CPU @ 3.00 GHz, 16GB RAM ).}

It is intuitive that the execution time of HARR-V and HARR-M increases linearly, which verifies the time complexity analyzed in Theorem~\ref{the:timepp}. Moreover, the time curves of both methods exhibit nearly identical trends, further validating their practical performance and consistency with the theoretical expectations. This consistency demonstrates the scalability of the proposed methods, making them well-suited for handling large-scale datasets. In addition, it can be seen that even though HARR involves the weights learning procedure, its execution time is still shorter than that of CMS, UDM, HDM, and FBD. This is because the similarity computation cost dominates the execution time for all the compared approaches. {HARR directly adopts the CPDs’ differences as the base distance measure and is thus relatively more efficient than the similarity measurement strategies adopted by the other three metrics.}

\begin{figure}[!t]	
 
\centerline{\includegraphics[width=3.6in]{ 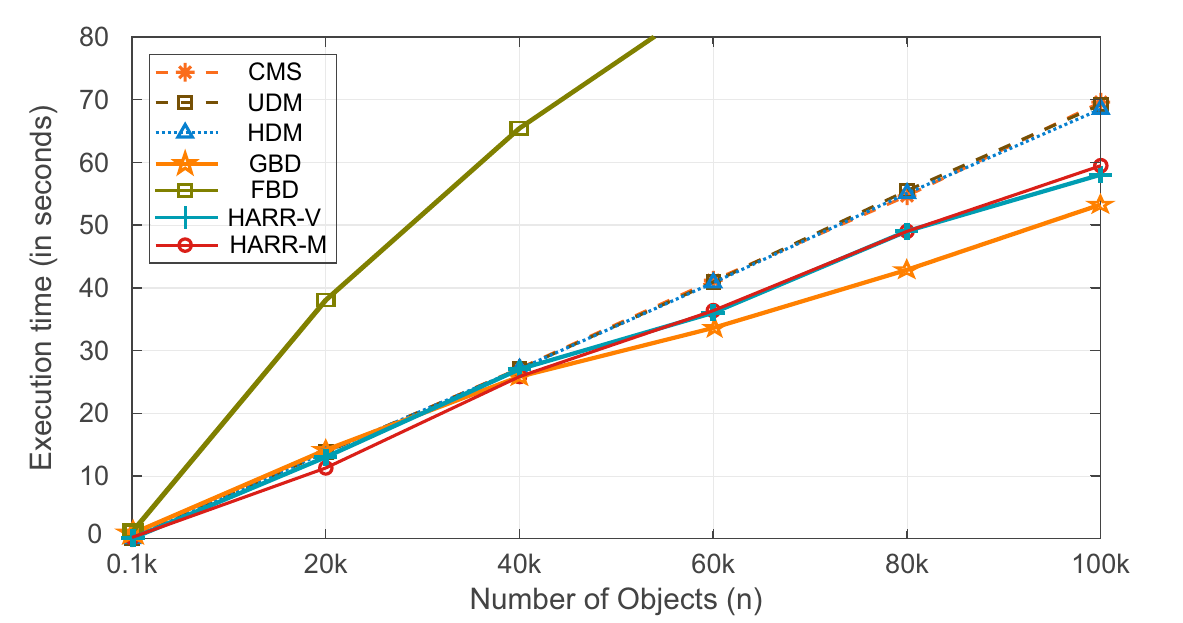}}  
\caption{{Comparison of the execution time of CMS, UDM, HDM, GBD, FBD, HARR-V, and HARR-M on a large synthetic dataset ($n = 100,000$) sampled at different sampling rates $\varphi =\{0.001, 0.2, 0.4, 0.6, 0.8, 1\}$.}}	
\label{fig:time_n}
\end{figure}

\subsection{Intuitive Study} \label{sct:vs}

{To showcase improvements of HARR-V and HARR-M and the intuitiveness of the distance structure they obtain, with a practical case. We use the distance between attribute values learned by OHE, GBD, FBD, HARR-V and HARR-M to encode the attributes of the MR dataset. The encoded data are then dimensionally reduced into a 2-D space through t-SNE~\cite{r1visual} and visualized in Figure \ref{fig:tsne} by marking the data points with `true' labels provided by the dataset in different colors. If more data points with the same label are gathered in the visualization, then it indicates that the corresponding distance metric is more competent in discriminating different clusters. }

{
It is evident from the figure that the cluster discrimination ability of the HARR-V and HARR-M methods is significantly better than that of OHE, GBD, and FBD. Specifically, compared to the methods shown in Figures~\ref{fig:tsne} (a)-(c), Figures~\ref{fig:tsne} (d) and (e) demonstrate that the distance structures learned by HARR-V and HARR-M enable the encoded MR dataset to exhibit more distinct cluster separation. This advantage stems from the projection mechanism employed by the HARR framework, which effectively preserves the information between categorical attributes. Additionally, by automating the assignment of attribute weights, the HARR methods adequately consider both inter-cluster and intra-cluster relationships, leading to more reasonable and discriminative representations. }

These improvements are particularly valuable in real-world scenarios where accurate and efficient clustering of heterogeneous data is critical, such as in market segmentation, healthcare diagnostics, or recommendation systems. In these domains, being able to distinguish clusters with more precision and interpretability can significantly enhance decision-making and result in more actionable insights. {For example, in market segmentation, the ability to capture subtle differences between customer groups based on categorical attributes can lead to more targeted marketing strategies.
}
\begin{figure}[h]	
\begin{minipage}{0.19\linewidth}	
  \centerline{\includegraphics[width=1.3in]{ 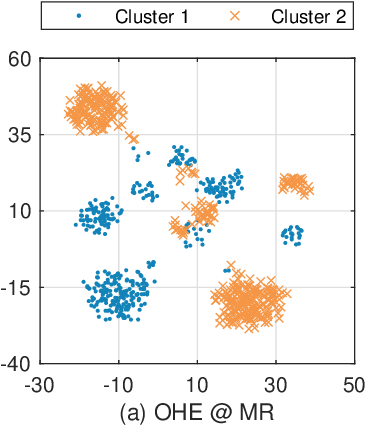}}
\end{minipage}	
\hfill	
\begin{minipage}{0.19\linewidth}	
  \centerline{\includegraphics[width=1.3in]{ 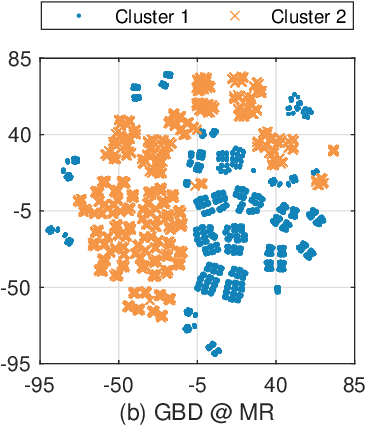}}
\end{minipage}
\hfill	
\begin{minipage}{0.19\linewidth}	
  \centerline{\includegraphics[width=1.3in]{ 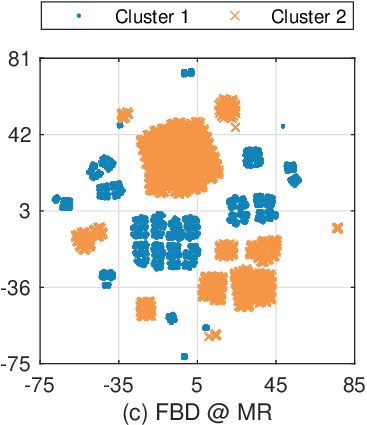}}
\end{minipage}
\hfill	
\begin{minipage}{0.19\linewidth}	
  \centerline{\includegraphics[width=1.3in]{ 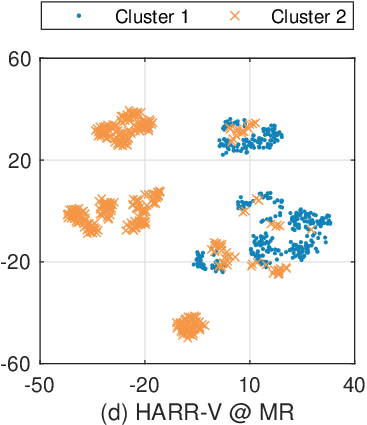}}
\end{minipage}
\hfill	
\begin{minipage}{0.19\linewidth}	
  \centerline{\includegraphics[width=1.3in]{ 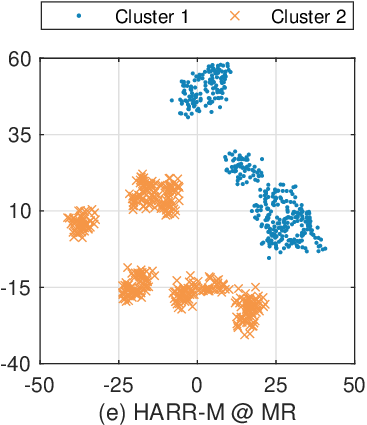}}
\end{minipage}
\caption{t-SNE visualization of the MR dataset represented by OHE, GBD, FBD, HARR-V, and HARR-M. {Data points are marked in different colors according to the `true' labels.} A higher concentration of data points with the same label indicates that the corresponding method's distance metric better discriminates between clusters.}
\label{fig:tsne}	
\end{figure}	

\section{Concluding Remarks}\label{sct:conclusion}

This paper proposes an attribute representation learning approach for solving the clustering problem of mixed data composed of heterogeneous attributes. Based on the analysis of the concepts described by the attributes, the values of each attribute are represented in homogeneous linear distance spaces. Such homogeneous representation provides an effective basis for the fusion of heterogeneous information from the different types of attributes in mixed data clustering. To make the representations adapt to clustering tasks, we propose a clustering paradigm to simultaneously search the weights of represented attributes and partitions of data objects. {The designed learning mechanisms can effectively circumvent sub-optimal solutions to a certain extent and demonstrate superior} clustering performance on both categorical data and more challenging mixed data. Moreover, the proposed clustering algorithms converge quickly and do not involve non-trivial parameter settings.

Despite the effectiveness of our method on static datasets, certain limitations remain. Specifically, our method tends to underperform when applied to datasets containing missing or noisy values. Additionally, in the context of complex and dynamic environments involving streaming data, {the  multiple spaces} projection mechanism employed by our approach may struggle to maintain its effectiveness. Addressing these issues will be a key focus of our future research. In particular, incorporating mechanisms for handling noisy or incomplete data, as well as developing incremental metric fusion techniques for streaming data under dynamic conditions, {are} promising directions for extending and improving our method.

\section*{Acknowledgements}
This work was supported in part by the National Natural Science Foundation of China (NSFC) under grants: 62476063, 62376233, and 62102097, the NSFC/Research Grants Council (RGC) Joint Research Scheme under the grant N\_HKBU214/21, the Natural Science Foundation of Guangdong Province under grant 2023A1515012855, the General Research Fund of RGC under grants: 12201321, 12202622, and 12201323, and the RGC Senior Research Fellow Scheme under grant SRFS2324-2S02.

\bibliographystyle{elsarticle-num-names}
\bibliography{LT_Learn_Rep_Mix_Clustering}

@book{intro1,
  title={Categorical data analysis},
  author={Agresti, Alan},
  year={2003},
  publisher={John Wiley \& Sons}
}

@article{intro5,
  title={Data clustering: A review},
  author={Jain, Anil K and Murty, M Narasimha and Flynn, Patrick J},
  journal={ACM Computing Surveys},
  volume={31},
  number={3},
  pages={264--323},
  year={1999},
  publisher={ACM}
}

@inproceedings{intro6,
  title={Similarity measures for categorical data: A comparative evaluation},
  author={Boriah, Shyam and Chandola, Varun and Kumar, Vipin},
  booktitle={Proceedings of the 2008 SIAM International Conference on Data Mining},
  pages={243--254},
  year={2008},
}

@book{intro9,
  title={Information theory and statistics},
  author={Kullback, Solomon},
  year={1997},
  publisher={Courier Corporation}
}

@article{intro10,
  title={Categorical data clustering: What similarity measure to recommend?},
  author={dos Santos, Tiago RL and Z{\'a}rate, Luis E},
  journal={Expert Systems with Applications},
  volume={42},
  number={3},
  pages={1247--1260},
  year={2015},
  publisher={Elsevier}
}

@article{kms,
  title={A clustering technique for summarizing multivariate data},
  author={Ball, Geoffrey H and Hall, David J},
  journal={Behavioral Science},
  volume={12},
  number={2},
  pages={153--155},
  year={1967},
  publisher={Wiley Online Library}
}

@article{kmd,
  title={Extensions to the k-means algorithm for clustering large data sets with categorical values},
  author={Huang, Zhexue},
  journal={Data Mining and Knowledge Discovery},
  volume={2},
  number={3},
  pages={283--304},
  year={1998},
  publisher={Springer}
}

@inproceedings{kpt,
title={Clustering large data sets with mixed numeric and categorical values},
author={Huang, Zhexue},
booktitle={Proceedings of the First Pacific-Asia Conference on Knowledge Discovery and Data Mining},
pages={21--34},
year={1997},
}

@article{wkm,
  title={Automated variable weighting in k-means type clustering},
  author={Huang, Joshua Zhexue and Ng, Michael K and Rong, Hongqiang and Li, Zichen},
  journal={IEEE Transactions on Pattern Analysis and Machine Intelligence},
  volume={27},
  number={5},
  pages={657--668},
  year={2005},
  publisher={IEEE}
}

@article{ewkm,
  title={An entropy weighting k-means algorithm for subspace clustering of high-dimensional sparse data},
  author={Jing, Liping and Ng, Michael K and Huang, Joshua Zhexue},
  journal={IEEE Transactions on Knowledge and Data Engineering},
  volume={19},
  number={8},
  year={2007},
  publisher={IEEE}
}

@article{mwkm,
  title={A novel attribute weighting algorithm for clustering high-dimensional categorical data},
  author={Bai, Liang and Liang, Jiye and Dang, Chuangyin and Cao, Fuyuan},
  journal={Pattern Recognition},
  volume={44},
  number={12},
  pages={2843--2861},
  year={2011},
  publisher={Elsevier}
}

@article{inikms,
  title={How much can k-means be improved by using better initialization and repeats?},
  author={Fr{\"a}nti, Pasi and Sieranoja, Sami},
  journal={Pattern Recognition},
  volume={93},
  pages={95--112},
  year={2019},
  publisher={Elsevier}
}

@article{oc,
  title={Categorical-and-numerical-attribute data clustering based on a unified similarity metric without knowing cluster number},
  author={Cheung, Yiu-ming and Jia, Hong},
  journal={Pattern Recognition},
  volume={46},
  number={8},
  pages={2228--2238},
  year={2013},
  publisher={Elsevier}
}

@article{scc,
  title={Soft subspace clustering of categorical data with probabilistic distance},
  author={Chen, Lifei and Wang, Shengrui and Wang, Kaijun and Zhu, Jianping},
  journal={Pattern Recognition},
  volume={51},
  pages={322--332},
  year={2016},
  publisher={Elsevier}
}

@article{woc,
  title={Subspace clustering of categorical and numerical data with an unknown number of clusters},
  author={Jia, Hong and Cheung, Yiu-ming},
  journal={IEEE Transactions on Neural Networks and Learning Systems},
  volume={29},
  number={8},
  pages={3308--3325},
  year={2018},
  publisher={IEEE}
}

@inproceedings{dlc,
  title={An ordinal data clustering algorithm with automated distance learning},
  author={Zhang, Yiqun and Cheung, Yiu-ming},
  booktitle={Proceedings of the 34th AAAI Conference on Artificial Intelligence},
  pages={6869--6876},
  year={2020},
}

@article{gsm,
  title={A new similarity index based on probability},
  author={Goodall, David W},
  journal={Biometrics},
  pages={882--907},
  year={1966},
  publisher={JSTOR}
}

@article{hdm,
  title={Studies in classification, data Analysis, and knowledge organization},
  author={Arabie, Ph and Baier, Newark D and Critchley, Cottbus F and Keynes, Milton},
  year={2006},
  publisher={Springer}
}

@inproceedings{lsm,
  title={An information-theoretic definition of similarity},
  author={Lin, Dekang},
  booktitle={Proceedings of the 15th International Conference on Machine Learning},
  pages={296--304},
  year={1998}
}

@article{abdm,
  title={An association-based dissimilarity measure for categorical data},
  author={Le, Si Quang and Ho, Tu Bao},
  journal={Pattern Recognition Letters},
  volume={26},
  number={16},
  pages={2549--2557},
  year={2005},
  publisher={Elsevier}
}

@article{adm,
  title={A method to compute distance between two categorical values of same attribute in unsupervised learning for categorical data set},
  author={Ahmad, Amir and Dey, Lipika},
  journal={Pattern Recognition Letters},
  volume={28},
  number={1},
  pages={110--118},
  year={2007},
}

@inproceedings{cbdm_conf,
  title={Context-based distance learning for categorical data clustering},
  author={Ienco, Dino and Pensa, Ruggero G and Meo, Rosa},
  booktitle={Proceedings of the Eighth International Symposium on Intelligent Data Analysis},
  pages={83--94},
  year={2009}
}

@article{cbdm_journal,
  title={From context to distance: Learning dissimilarity for categorical data clustering},
  author={Ienco, Dino and Pensa, Ruggero G and Meo, Rosa},
  journal={ACM Transactions on Knowledge Discovery from Data},
  volume={6},
  number={1},
  pages={1--25},
  year={2012},
  publisher={ACM}
}

@article{jdm,
  title={A new distance metric for unsupervised learning of categorical data},
  author={Jia, Hong and Cheung, Yiu-ming and Liu, Jiming},
  journal={IEEE Transactions on Neural Networks and Learning Systems},
  volume={27},
  number={5},
  pages={1065--1079},
  year={2016},
  publisher={IEEE}
}

@inproceedings{ebdmconf,
  title={Exploiting order information embedded in ordered categories for ordinal data clustering},
  author={Zhang, Yiqun and Cheung, Yiu-ming},
  booktitle={Proceedings of the 24th International Symposium on Methodologies for Intelligent Systems},
  pages={247--257},
  year={2018}
}

@article{ebdmjournal,
  title={A unified entropy-based distance metric for ordinal-and-nominal-attribute data clustering},
  author={Zhang, Yiqun and Cheung, Yiu-ming and Tan, Kaychen},
  journal={IEEE Transactions on Neural Networks and Learning Systems},
  volume={31},
  number={1},
  pages={39--52},
  year={2020},
  publisher={IEEE}
}

@article{udm,
  title={A new distance metric exploiting heterogeneous inter-attribute relationship for ordinal-and-nominal-attribute data clustering},
  author={Zhang, Yiqun and Cheung, Yiu-ming},
  journal={IEEE Transactions on Cybernetics},
  volume={52},
  number={2},
  pages={758--771},
  year={2022},
  publisher={IEEE}
}

@article{cms,
  title={Unsupervised coupled metric similarity for non-IID categorical data},
  author={Jian, Songlei and Cao, Longbing and Lu, Kai and Gao, Hang},
  journal={IEEE Transactions on Knowledge and Data Engineering},
  volume={30},
  number={9},
  pages={1810--1823},
  year={2018},
  publisher={IEEE}
}

@article{sbc,
  title={Space structure and clustering of categorical data},
  author={Qian, Yuhua and Li, Feijiang and Liang, Jiye and Liu, Bing and Dang, Chuangyin},
  journal={IEEE Transactions on Neural Networks and Learning Systems},
  volume={27},
  number={10},
  pages={2047--2059},
  year={2015},
  publisher={IEEE}
}

@inproceedings{cde_conf,
  title={Embedding-based representation of categorical data by hierarchical value coupling learning},
  author={Jian, Songlei and Cao, Longbing and Pang, Guansong and Lu, Kai and Gao, Hang},
  booktitle={Proceedings of the 26th International Joint Conference on Artificial Intelligence},
  year={2017}
}

@article{cde,
  title={CURE: Flexible categorical data representation by hierarchical coupling learning},
  author={Jian, Songlei and Pang, Guansong and Cao, Longbing and Lu, Kai and Gao, Hang},
  journal={IEEE Transactions on Knowledge and Data Engineering},
  volume={31},
  number={5},
  pages={853--866},
  year={2018},
  publisher={IEEE}
}

@inproceedings{mai,
  title={Metric-based auto-instructor for learning mixed data representation},
  author={Jian, Songlei and Hu, Liang and Cao, Longbing and Lu, Kai},
  booktitle={Proceedings of the 32nd AAAI Conference on Artificial Intelligence},
  pages={3318--3325},
  year={2018}
}

@article{untie,
  title={Unsupervised heterogeneous coupling learning for categorical representation},
  author={Zhu, Chengzhang and Cao, Longbing and Yin, Jianping},
  journal={IEEE Transactions on Pattern Analysis and Machine Intelligence},
  volume={44},
  number={1},
  pages={553--549},
  year={2022},
  publisher={IEEE}
}

@inproceedings{mix2vec,
  title={Mix2Vec: Unsupervised Mixed Data Representation},
  author={Zhu, Chengzhang and Zhang, Qi and Cao, Longbing and Abrahamyan, Arman},
  booktitle={Proceedings of the seventh International Conference on Data Science and Advanced Analytics},
  pages={118--127},
  year={2020},
}

@article{hd-ndw,
  title={Learnable Weighting of Intra-attribute Distances for Categorical Data Clustering with Nominal and Ordinal Attributes},
  author={Zhang, Yiqun and Cheung, Yiu-ming},
  journal={IEEE Transactions on Pattern Analysis and Machine Intelligence},
  volume={44},
  number={7},
  pages={3560--3576},
  year={2022},
  publisher={IEEE}
}

@inproceedings{het2hom,
  title={Het2Hom: Representation of heterogeneous attributes into homogeneous concept spaces for categorical-and-numerical-attribute data clustering},
  author={Zhang, Yiqun and Cheung, Yiu-ming and Zeng, An},
  booktitle={Proceedings of the 31st International Joint Conference on Artificial Intelligence},
  pages={3758--3765},
  year={2022}
}

@article{sig_clust,
  title={Significance-Based Categorical Data Clustering},
  author={Hu, Lianyu and Jiang, Mudi and Liu, Yan and He, Zengyou},
  journal={arXiv preprint arXiv:2211.03956},
  year={2022}
}

@misc{uci,
author = {Dua, Dheeru and Karra Taniskidou, Efi},
year = {2017},
title = {{UCI} Machine Learning Repository},
url = {http://archive.ics.uci.edu/ml},
institution = {University of California, Irvine, School of Information and Computer Sciences}
}

@inproceedings{ex1,
  title={Laplacian score for feature selection},
  author={He, Xiaofei and Cai, Deng and Niyogi, Partha},
  booktitle={Proceedings of the 18th International Conference on Neural Information Processing Systems},
  pages={507--514},
  year={2005}
}

@inproceedings{ex2,
  title={On the use of the adjusted rand index as a metric for evaluating supervised classification},
  author={Santos, Jorge M and Embrechts, Mark},
  booktitle={Proceedings of the 19th International Conference on Artificial Neural Networks},
  pages={175--184},
  year={2009},
  organization={Springer}
}

@article{ex3,
  title={Objective criteria for the evaluation of clustering methods},
  author={Rand, William M},
  journal={Journal of the American Statistical Association},
  volume={66},
  number={336},
  pages={846--850},
  year={1971},
  publisher={Taylor \& Francis Group}
}

@article{ex4,
  title={The impact of random models on clustering similarity},
  author={Gates, Alexander J and Ahn, Yong-Yeol},
  journal={The Journal of Machine Learning Research},
  volume={18},
  number={1},
  pages={3049--3076},
  year={2017},
  publisher={JMLR. org}
}

@article{ex8,
  title={Fusing monotonic decision trees},
  author={Qian, Yuhua and Xu, Hang and Liang, Jiye and Liu, Bing and Wang, Jieting},
  journal={IEEE Transactions on Knowledge and Data Engineering},
  volume={27},
  number={10},
  pages={2717--2728},
  year={2015},
  publisher={IEEE}
}

@article{r1test,
  title={Statistical comparisons of classifiers over multiple data sets},
  author={Dem{\v{s}}ar, Janez},
  journal={Journal of Machine Learning Research},
  volume={7},
  number={1},
  pages={1--30},
  year={2006}
}

@article{r1visual,
  title={Visualizing data using t-SNE},
  author={Maaten, Laurens van der and Hinton, Geoffrey},
  journal={Journal of Machine Learning Research},
  volume={9},
  number={11},
  pages={2579--2605},
  year={2008}
}

@article{ML23KBS1,
title = {A Multi-View Deep Metric Learning approach for Categorical Representation on mixed data},
author = {Qiude Li and Shengfen Ji and Sigui Hu and Yang Yu and Sen Chen and Qingyu Xiong and Zhu Zeng},
journal = {Knowledge-Based Systems},
volume = {260},
pages = {110161},
year = {2023}
}

@article{YU2022KBS4,
title = {Multi-view distance metric learning via independent and shared feature subspace with applications to face and forest fire recognition, and remote sensing classification},
journal = {Knowledge-Based Systems},
author = {Yifan Yu and Liyong Fu and Yawen Cheng and Qiaolin Ye},
volume = {243},
pages = {108350},
year = {2022}
}

@article{KHAN23KBS5,
title = {An entropy-based weighted dissimilarity metric for numerical data clustering using the distribution of intra feature differences},
author = {Abdul Atif Khan and Amaresh Chandra Mishra and Sraban Kumar Mohanty},
journal = {Knowledge-Based Systems},
volume = {280},
pages = {110967},
year = {2023}

}

@ARTICLE{Adc,
  author={Zhang,Yiqun and Cheung,Yiu-Ming },
  journal={IEEE Transactions on Neural Networks and Learning Systems}, 
  title={Graph-Based Dissimilarity Measurement for Cluster Analysis of Any-Type-Attributed Data}, 
  year={2023},
  volume={34},
  number={9},
  pages={6530-6544},
 }

@inproceedings{COForest, 
  title={Learning Order Forest for Qualitative-Attribute Data Clustering}, 
  author={Zhao, Mingjie and Feng, Sen  and Zhang, Yiqun and Li, Mengke and Lu, Yang and Cheung, Yiu-Ming}, 
  booktitle={Proceedings of the 27th European Conference on Artificial Intelligence},
  year={2024}
}

@article{Gower,
 author = {J. C. Gower},
 journal = {Biometrics},
 number = {4},
 pages = {857--871},
 title = {A General Coefficient of Similarity and Some of Its Properties},
 urldate = {2024-11-30},
 volume = {27},
 year = {1971}
}

@article{friedman,
  title={The use of ranks to avoid the assumption of normality implicit in the analysis of variance},
  author={Friedman, Milton},
  journal={Journal of the american statistical association},
  volume={32},
  number={200},
  pages={675--701},
  year={1937},
  publisher={Taylor \& Francis}
}

\end{document}